\documentclass{article}
\usepackage[accepted]{aistats2e}

\usepackage{natbib,url}
\usepackage{graphicx} 
\usepackage{subfigure}
\usepackage{amssymb,amsthm}
\usepackage{multirow}

\theoremstyle{plain} \newtheorem{thm}{Theorem}
 \newtheorem{lem}[thm]{Lemma}
\newtheorem{prop}[thm]{Proposition}

\theoremstyle{definition} \newtheorem{defn}{Definition}

\theoremstyle{remark} 

\begin{document}

%

%

\twocolumn[

\aistatstitle{Spectral Clustering with Unbalanced Data}

\aistatsauthor{ Jing Qian \And Venkatesh Saligrama}

\aistatsaddress{ Boston University \And Boston University } ]

\begin{abstract}
Spectral clustering (SC) and graph-based semi-supervised learning (SSL) algorithms are sensitive to how graphs are constructed from data. In particular if the data has proximal and unbalanced clusters these algorithms can lead to poor performance on well-known graphs such as $k$-NN, full-RBF, $\epsilon$-graphs. This is because the objectives such as Ratio-Cut (RCut) or normalized cut (NCut) attempt to tradeoff cut values with cluster sizes, which are not tailored to unbalanced data.
We propose a novel graph partitioning framework, which parameterizes a family of graphs by adaptively modulating node degrees in a $k$-NN graph. We then propose a model selection scheme to choose sizable clusters which are separated by smallest cut values.
Our framework is able to adapt to varying levels of unbalancedness of data and can be naturally used for small cluster detection. We theoretically justify our ideas through limit cut analysis. Unsupervised and semi-supervised experiments on synthetic and real data sets demonstrate the superiority of our method.
\end{abstract}

\section{Introduction}\label{sec:intro_motiv}
Data with unbalanced clusters arises in many learning applications and has attracted much interest \citep{HeGarcia09}.
In this paper we focus on graph-based spectral methods for clustering and semi-supervised learning (SSL) tasks.
While model-based approaches \citep{Fraley02} may incorporate unbalancedness, they typically assume simple cluster shapes and need multiple restarts. In contrast non-parametric graph-based approaches do not have this issue and are able to capture complex shapes \citep{Ng01}. In spectral methods a graph representing data is first constructed. Then a graph-based learning algorithm such as spectral clustering(SC) \citep{Hagen92,Shi00} or SSL algorithms \citep{Zhu08,WanJebCha08} is applied on the graph. Of the two steps, graph construction has been identified to be important \citep{Luxburg07,Maier1,JebWanCha09}, and we will see is critical in the presence of unbalanced proximal clusters. Common graph construction methods include $\epsilon$-graph, fully-connected RBF-weighted(full-RBF) graph and $k$-nearest neighbor($k$-NN) graph. Of the three $k$-NN graphs appears to be most popular due to its relative robustness to outliers \citep{Zhu08,Luxburg07}. 

Drawbacks of spectral methods on unbalanced data have been documented: \citet{Zelnik04} suggests an adaptive RBF parameter for full-RBF graph. More recently, \citet{Nadler07} describe these drawbacks from a random walk perspective. Nevertheless, to the best of our knowledge, there does not exist systematic ways of adapting spectral methods to possibly unbalanced data.
There are other spectral methods \citep{Buhler09,ShiBelkinYu09} that are claimed to be able to handle unbalanced clusters better than standard SC. However, they do not look into unbalanced data specifically; meanwhile our framework can be combined with these methods.
Also related is size-constrained clustering \citep{ST97,Hoppner08,ZhuWangLi10} which imposes constraints on the number of points per cluster. This is a different problem because with size constraints the partitions may not be low-density cuts, while our clustering goal here is to find natural partitions separated by density valleys -- clusters could be unbalanced but we do not know a priori how unbalanced they are.


The poor performance of spectral methods in the presence of unbalanced clusters is a result of minimizing RatioCut (RCut) or normalized cut (NCut) objective on these graphs, which seeks a tradeoff between minimum cut-values and cluster sizes. While robust to outliers, this sometimes leads to meaningless cuts.
In Section 2 we illustrate some of the fundamental issues underlying poor performance of spectral methods on unbalanced data. We then describe a novel graph-based learning framework in Section 3.
Specifically we propose to parameterize a family of graphs by adaptively modulating node degrees based on the ranking of all samples.
This rank-modulated degree (RMD) strategy asymptotically results in reduced (increased) cut-values near density valleys (high-density areas).
Based on this parametric scheme we present a model selection step that finds the lowest-density partition with sizable clusters.
Our approach is able to handle varying levels of unbalanced data and detect small clusters.
We explore the theoretical basis in Section~4. In Section~5 we present experiments on synthetic and real datasets to show significant improvements in SC and SSL results over conventional graphs. Proofs appear in supplementary section.

\section{Problem Definition}\label{sec:motiv}
We describe an abstract continuous setting to describe our problem.
Assume that data is drawn from some unknown density $f(x)$, where $x \in \mathbb{R}^d$. For simplicity we consider binary clustering problems but our setup generalizes to arbitrary number of partitions. We seek a hypersurface $S$ that partitions $\mathbf{R}^d$ into two non-empty subsets $D$ and $\bar D$ (with $D \cup \bar D=\mathbf{R}^d$).

While there are many ways to formulate partitioning problems we formulate the goal of binary partitioning to find a hypersurface that passes through minimum density regions, namely,
\begin{equation} \label{e.optim}
S_0 = \mbox{arg}\min_{S} \int_S \psi(f(s)) ds
\end{equation}
where $\psi(\cdot)$ is some positive monotonic function. This goal is too simplistic that the resulting partitions could be empty. Consequently, we need to constrain the measures, $\min\{\mu(D),\mu(\bar D)\} \geq \delta$ for some $\delta > 0$, to ensure meaningful partitions, where $\mu(A) = Prob\{x \in A\}$.
Certainly the optimal hypersurface $S_0$ may not necessarily be balanced.
\begin{defn}
We say the data is $\alpha$-unbalanced if the hypersurface $S_0$ results in partitions, $(D_0,\bar D_0)$, with:
\[
    \min\{\mu(D_0),\mu(\bar D_0)\} = \alpha < 1/2.
\]
\end{defn}

We now focus on finite sample objective mirroring the continuous objective of Eq.(\ref{e.optim}).
Let $G=(V,E)$ be a graph constructed using $n$ samples in some manner consistent with the underlying topology of the ambient space. We denote by $S$ a cut that partitions $V$ into $C_S$ and $\bar C_S$. The cut-value associated with $S$ is:
\begin{eqnarray}\label{equ:cut}
Cut(C_S,\bar{C_S}) = \sum_{u\in C_S,v\in \bar{C_S},(u,v)\in E}w(u,v)
\end{eqnarray}
The empirical variant of Eq.(\ref{e.optim}) is to minimize the cut-value subject to sizable cluster constraints:
{\small
\begin{equation} \label{e.empopt}
S_* = \mbox{arg} \min_{S} \left \{ Cut(C_S,\bar C_S) \mid \min\{|C_S|,\,|\bar C_S|\} \geq \delta |V| \right \}
\end{equation}
}
\noindent
We assume that the cut $S_*$ results in $\left(C_*,\,\bar C_*\right)$.

\subsection{Graph Partitioning Algorithms}\label{subsec:alg}

Existing graph partitioning algorithms aim to minimize various objectives on the graph. The min-cut approach \citep{Stoer97} directly minimizes the cut-value Eq.(\ref{equ:cut}). While simple and efficient, this method could suffer from serious outlier problems without sizable cluster constraints. The popular SC algorithms attempt to minimize RCut or NCut:
\begin{eqnarray}\label{equ:ratiocut}
    RCut(C_S,\bar{C_S})=Cut(C_S,\bar{C_S}) \left( \frac{|V|}{|C_S|} + \frac{|V|}{|\bar{C_S}|} \right),
 \end{eqnarray}
$$
NCut(C_S,\bar{C_S})=Cut(C_S,\bar{C_S})\left(\frac{vol(V)}{vol(C_S)}+\frac{vol(V)}{vol(\bar{C_S})}\right),
$$
where $vol(C)=\sum_{u\in C,v\in V}w(u,v)$. Both NCut and RCut seek to trade-off low cut-values against cut size.
While robust to outliers, minimizing RCut(NCut) on traditional graphs can fail when data is unbalanced (i.e. with small $\alpha$ of Def.1).
To further motivate this issue, we first define two quantities: cut-ratio, $q$, and unbalancedness coefficient, $y$, associated with optimal cuts resulting from Eq.(\ref{e.empopt}) and any balanced cut $S_B$:
$$
q = {Cut(C_*,\bar{C_*}) \over Cut(C_B,\bar{C_B})};\,\,\,y= \frac{\min \{size(C_*),\,size(\bar{C_*})\}}{size(C_*)+size(\bar{C_*})},
$$
where $size(C)=|C|$ for RCut and $size(C)=vol(C)$ for NCut. The cut $S_B$ is associated with $\left(C_B,\,\bar C_B\right)$ and is balanced, i.e., $size(C_B)=size(\bar C_B)$. Note that $y \in [0,0.5]$ is an empirical measure of unbalancedness of the optimal cut $S_*$, while $q \in [0,1]$ is the proportion of cut-value at a ``density valley'' to that of a balanced cut.
Next we characterize the necessary condition for minimizing RCut(NCut) to work correctly.
\begin{prop}
SC fails, i.e., the RCut/NCut values of balanced cuts $S_B$ are smaller than that of $S_*$ (obtained in Eq.(\ref{e.empopt})), whenever $q > 4y(1-y)$.
\end{prop}
The proof follows by direct substitution.
Prop.1 suggests (see Fig.\ref{fig:qy}) that if the unbalancedness, $y$, is sufficiently small, say, $0.15$, then the cut value at the ``density valley'' has to be more than twice as deep for RCut/NCut to be effective.
\begin{figure}[tb]
\begin{centering}
\begin{minipage}[t]{.38\textwidth}
\includegraphics[width = 1\textwidth]{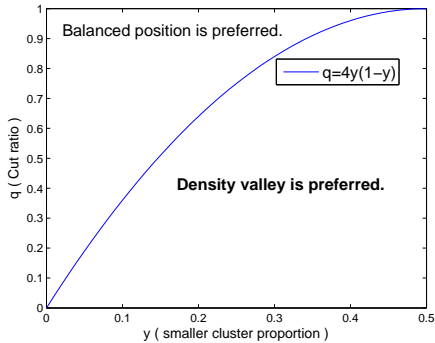}
\end{minipage}
\vspace{-0.1in}
\caption{\small Cut-ratio ($q$) vs unbalancedness ($y$). RCut value is smaller for balanced cuts than unbalanced low-density cuts whenever the cut-ratio is above the curve.}
\label{fig:qy}
\end{centering}
\vspace{-0.1in}
\end{figure}

Examining the limit behavior for $k$-NN, $\epsilon$-graph and full-RBF graphs \citep{Maier1,Narayanan06} is instructive to understand the pair $(q,y)$. For properly chosen $k_n$, $\sigma_n$ and $\epsilon_n$ respectively, as the number of samples $n\rightarrow \infty$, $q$ and $y$ converge with high probability to:

\vspace{-0.15in}
{\small
\begin{equation} \label{e.limval}
q \rightarrow {\int_{S_0} f^{\gamma}(x)dx \over \int_{S_B} f^{\gamma}(x)dx},\,\,\,\,\, y \rightarrow \min \{\mu(D_0),\,\mu(\bar D_0)\} = \alpha
\end{equation}}
\noindent
where $\gamma$ is a graph-dependent constant\footnote{For $k$-NN $\gamma < 1$ and $\gamma \in [1,2]$ for $\epsilon$ and full-RBF graphs.}. $S_0$ is the solution to Eq.(\ref{e.optim}) resulting in $\left(D_0,\bar D_0\right)$.
Essentially, each graph construction corresponds to a point $(q, y)$ on Fig.1. The issue with traditional graphs is that with $q$ unchanged, the data could be so unbalanced that $y$ falls above the curve on Fig.1, leading RCut(NCut) to pick the balanced cut!

\begin{figure*}[!tb]
\begin{centering}
\begin{minipage}[t]{.32\textwidth}
\includegraphics[width = 1\textwidth]{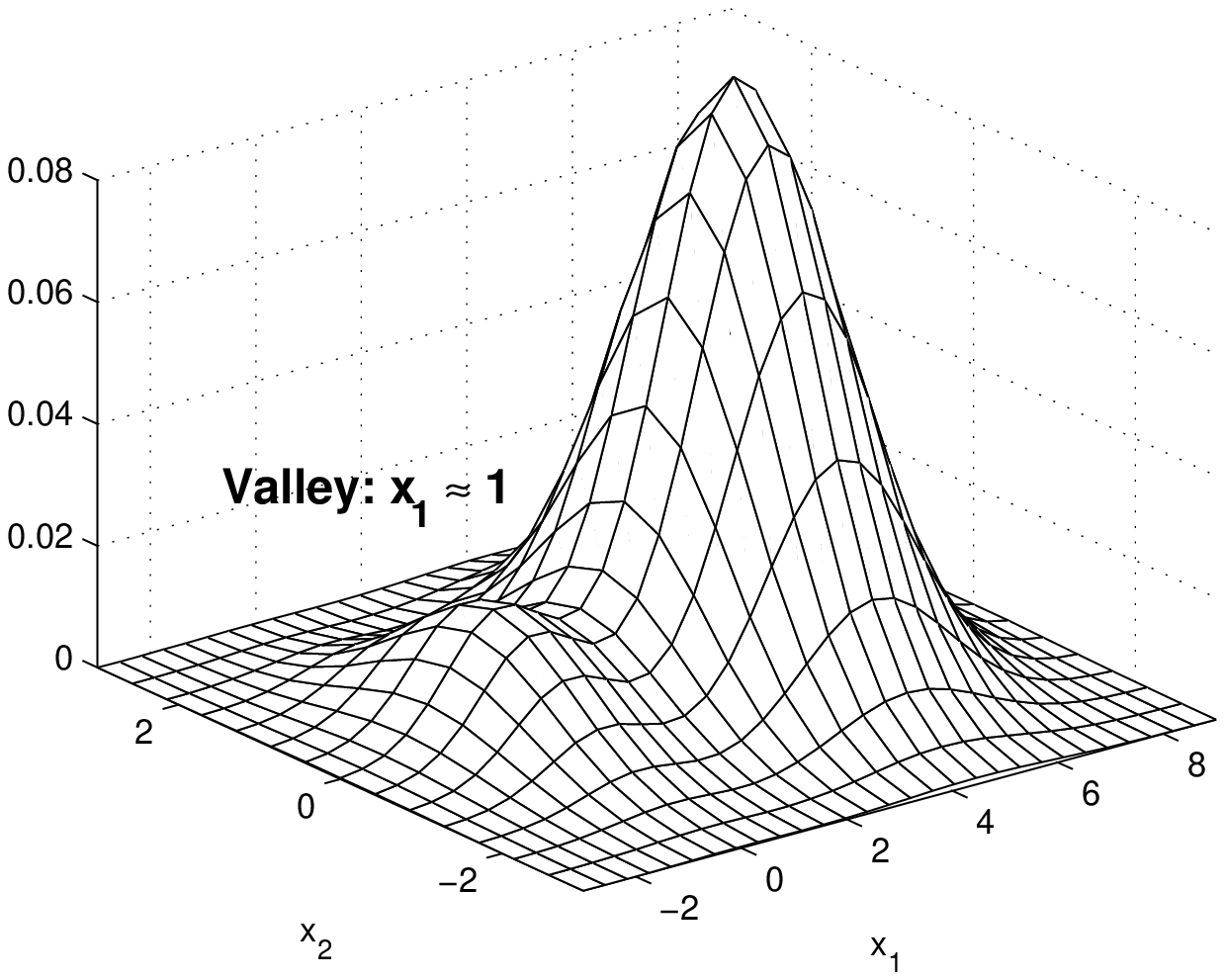}
\makebox[5.5 cm]{\small (a) pdf}
\end{minipage}
\begin{minipage}[t]{.32\textwidth}
\includegraphics[width = 1\textwidth]{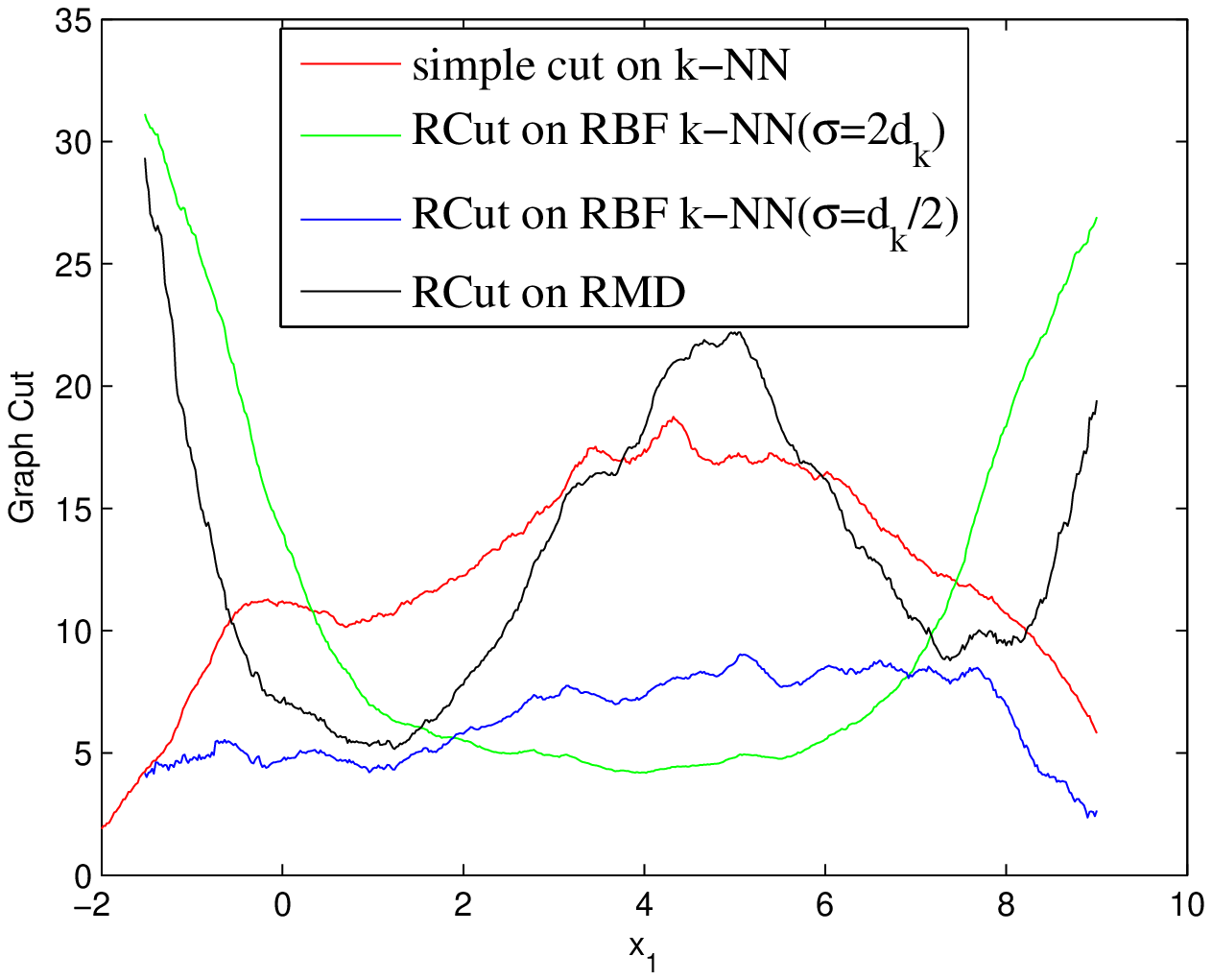}
\makebox[5.5 cm]{\small (b) RCut of $k$-NN and RMD }
\end{minipage}
\begin{minipage}[t]{.32\textwidth}
\includegraphics[width = 1\textwidth]{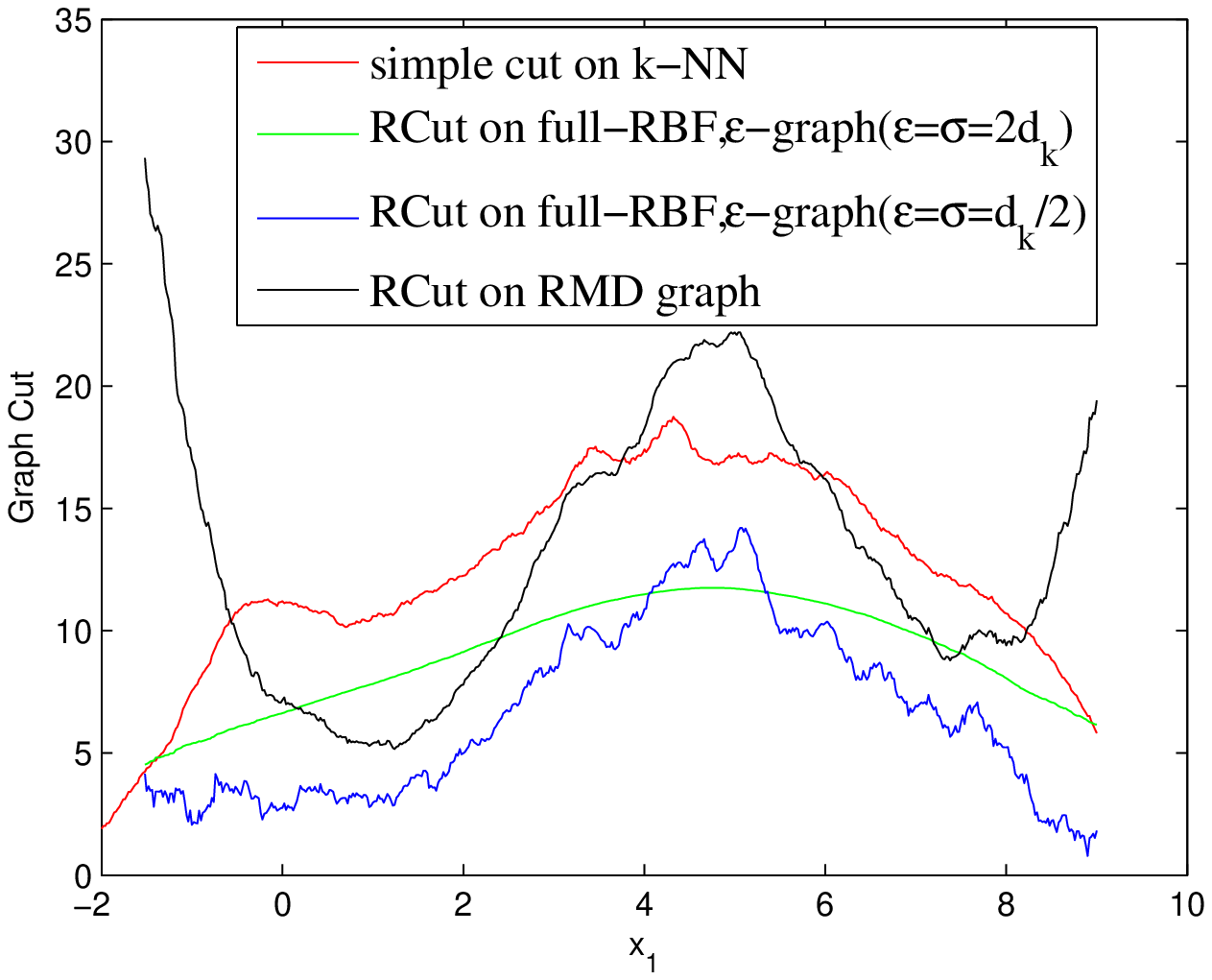}
\makebox[5.5 cm]{\small (c) RCut of full-RBF, $\epsilon$ and RMD }
\end{minipage}
\begin{minipage}[t]{.32\textwidth}
\includegraphics[width = 1\textwidth]{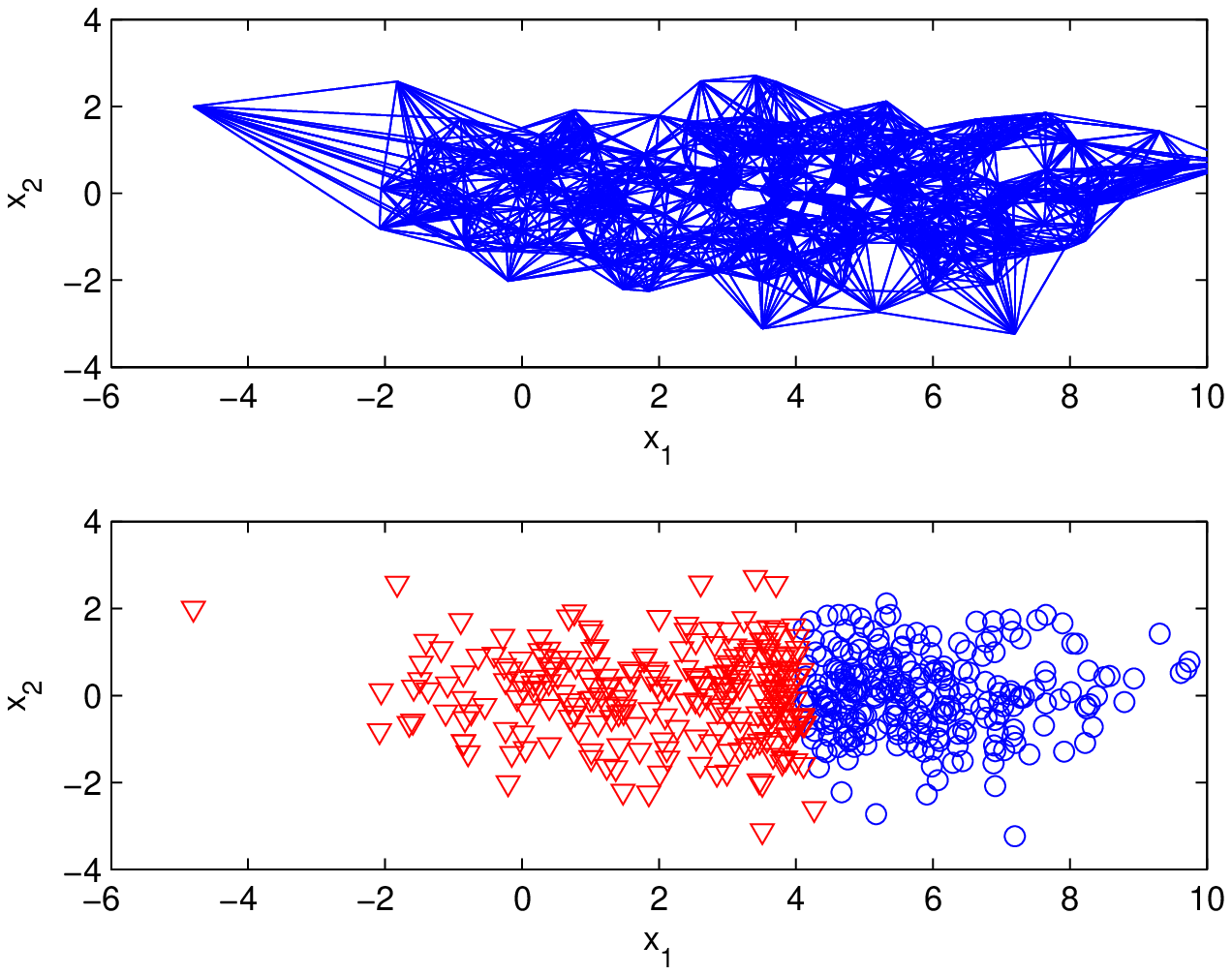}
\makebox[5.5 cm]{\small (d) $k$-NN}
\end{minipage}
\begin{minipage}[t]{.32\textwidth}
\includegraphics[width = 1\textwidth]{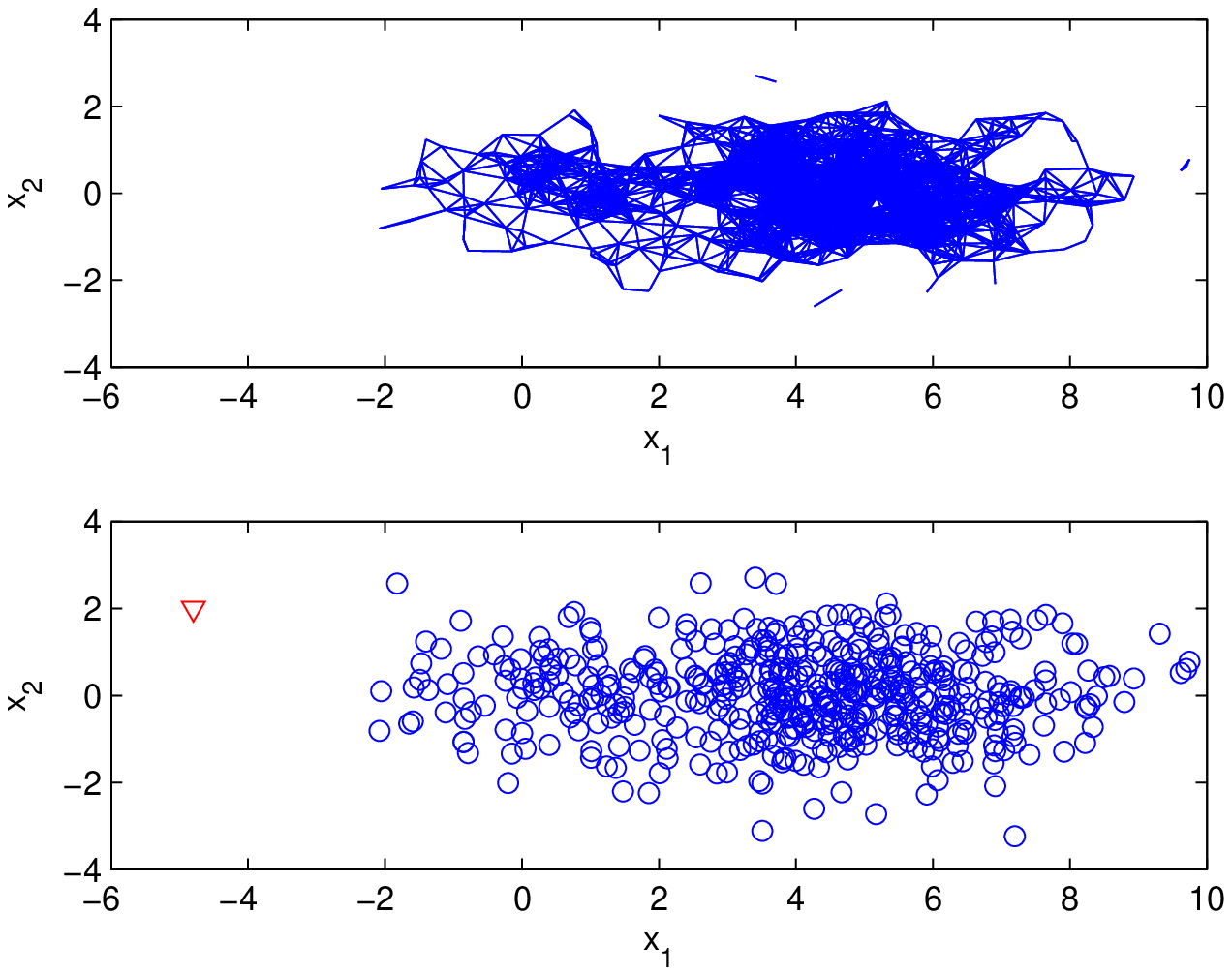}
\makebox[5.5 cm]{\small (e) full-RBF and $\epsilon$-graph}
\end{minipage}
\begin{minipage}[t]{.32\textwidth}
\includegraphics[width = 1\textwidth]{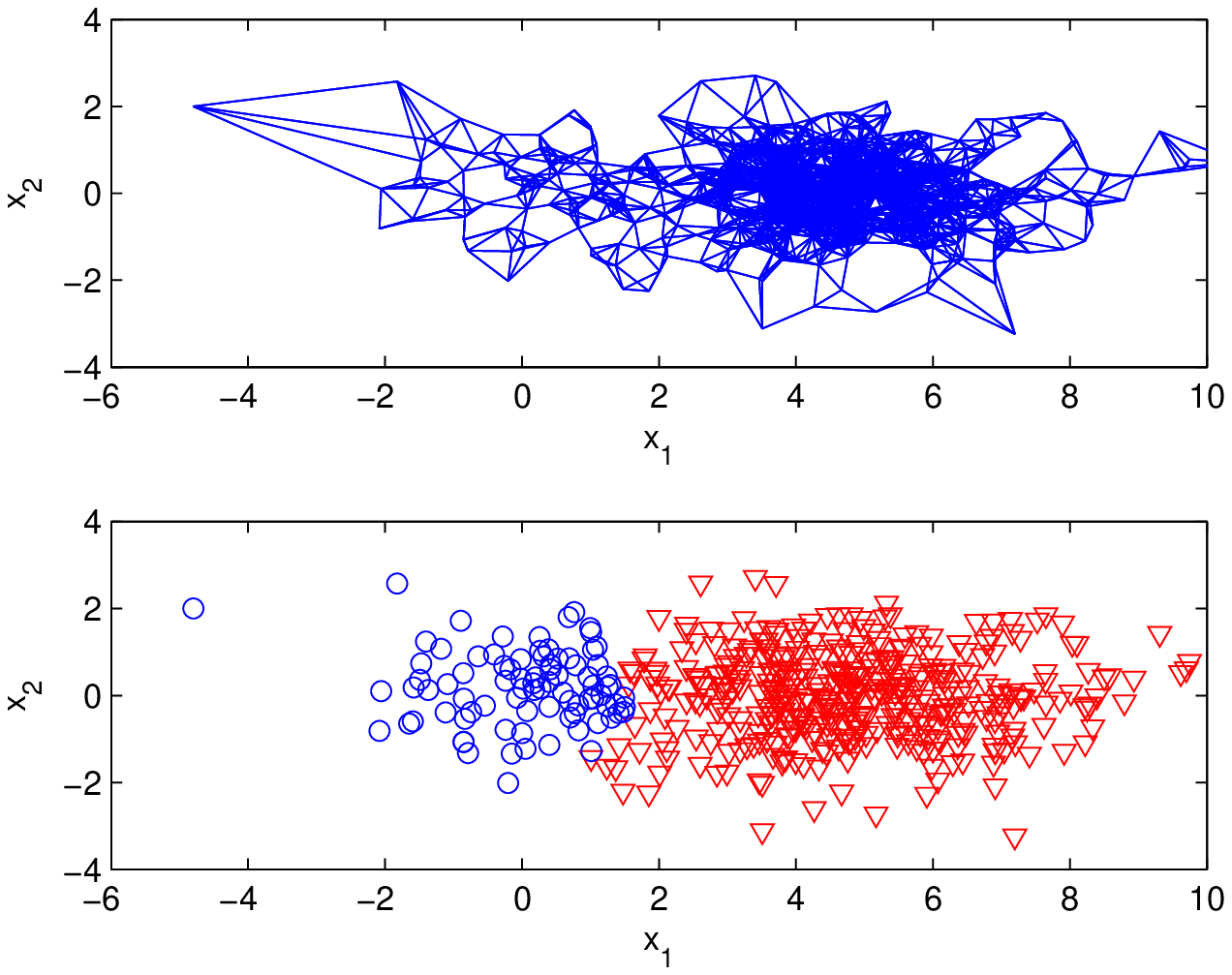}
\makebox[5.5 cm]{\small (f) RMD(our method)}
\end{minipage}
\caption{\small Mixture of two Gaussians with mixture proportions $0.85$ and $0.15$. The corresponding mean vectors are respectively $[4.5;0],\,\,[0;0]$,  and the covariances are $diag(2,1), \,\,\,diag(1,1)$. Cut and RCut values for $k$-NN, $\epsilon$ and full-RBF and our graph(RMD) are plotted. Figures in (b),(c) are averaged over 20 Monte Carlo runs. The values are re-scaled for demonstration. $\sigma$ is the RBF parameter. $d_k$ is the average $k$-NN distance. The number of samples is $n=1000$, and $k=30$. For (f) unweighted RMD graph with $l=30, \lambda=0.4$; for (d) unweighted $k$-NN; for (e) $\epsilon=\sigma=d_k$. The example (b) shows that large $\sigma$ in the $k$-NN graph results in smoothing of cut-values and the minimum RCut is not at the density valley. (c),(e) show that smaller $\epsilon,\sigma$ have pronounced sensitivity to outliers (RCut curve goes down near boundaries), while large $\epsilon,\sigma$ smoothen the RCut value.}
\label{fig:2g_graph}
\end{centering}
\end{figure*}

\subsection{Our Algorithm}\label{subsec:sol}
In order to adapt RCut(NCut) based algorithms to unbalanced data, our main point is that the ``balancing term"  is difficult to manipulate since it corresponds to the number of samples. On the other hand, the cut-ratio $q$ can be controlled by \emph{adaptively parameterizing different graphs on the same node set}.

To motivate this idea consider binary partitioning on a weighted graph $G_0 = (V,E_0,W_0)$. The node set $V$ is associated with samples, while $E_0, W_0$ are obtained through $k$-NN, full-RBF or $\epsilon$-graph or a combination thereof. The important point for future reference is that the graph $G_0=(V,E_0,W_0)$ is fixed. We seek a cut $(C, \bar C)$ for this graph that satisfies Eq.(\ref{e.empopt}). Conventional methods rely on a number of different graph partitioning techniques such as RCut/NCut based SC to obtain partitions. We have argued that this can lead to skewing the cut towards balanced cuts that are not representative of actual clusters.

To account for unbalancedness we adopt a new strategy here. The idea is to parameterize a family of graphs over a parametric space, $\lambda \in \Lambda$, with different edge sets on the same node set.
$$G(\lambda) = (V, E(\lambda),W(\lambda)),\,\, \lambda \in \Lambda$$
We will see that the mapping $E(\lambda),W(\lambda)$ allows for asymmetrical emphasis between low vs. high density for different choices of parameters. A number of graph partitioning techniques such as RCut/NCut based SC can now be applied on graphs with different choices of $\lambda \in \Lambda$ to obtain different partitions.
We can thus obtain a mapping from $\lambda \in \Lambda$ to a partition $\left(C(\lambda), \bar C(\lambda)\right)$,
$$
\lambda \longrightarrow (C(\lambda), \bar C(\lambda))
$$
Then we can evaluate the cut value of these partitions on the reference graph $G_0$, namely,
$$
Cut_0(C(\lambda),\bar C(\lambda)) = \sum_{u \in C(\lambda),\,v \in \bar C(\lambda)} w_0(u,v)  \mathbf{1}_{uv \in E_0}
$$
We can then pick the $\lambda$ (the partition) that minimizes the cut value under the constraint that each cluster has at least a $\delta$ fraction of the samples,
{\small
\begin{equation} \label{e.empopt1}
\lambda_* = \mbox{arg} \min_{\lambda \in \Lambda} \left \{ Cut_0(C(\lambda),\bar C(\lambda)) \mid \min\{|C(\lambda)|,\,|\bar C(\lambda)|\} \geq \delta |V| \right \}
\end{equation}
}
and output $\left(C(\lambda_*),\bar C(\lambda_*)\right)$ as the optimal partition. Notice our framework exactly aims at the optimal criterion Eq.\ref{e.empopt}).

This motivates how to parameterize a family of graphs to obtain rich binary partitioning structures: \\
({\bf 1}) Adaptively modulate the degree $k=k(x)$ node-wise based on $k$-NN graph, \\
({\bf 2}) the neighborhood size $\epsilon=\epsilon(x)$ based on $\epsilon$-graph.\\
Both strategies are somewhat equivalent. We adopt the first scheme since it is easier to explicitly control the number of edges to ensure a connected graph. Specifically, we propose to modulate node degrees of a $k$-NN graph through a parametric way based on rankings of all samples. This rank indicates whether a node lies near low/high density areas; therefore degree modulation can lead to fewer/more edges at low/high density regions.
Consequently, for the same $y$ with node set fixed, the cut-ratio $q$ is directly reduced, pulling down the point $(q,y)$ on Fig.1, for RCut(NCut) based algorithms to seek density valley cuts.


We propose a novel graph partitioning framework involving the following steps:\\
\noindent
(a) Parameterize a family of graphs with different edge sets on the same node set;\\
\noindent
(b) Minimize RCut(NCut) on this family of graphs to get a family of partitions;\\
\noindent
(c) Select the best partition that solves Eq.(\ref{e.empopt}).




\noindent
{\bf Remark:}
Note that one could also parameterize a family of $k$-NN, full-RBF or $\epsilon$-graphs with $k,\epsilon,\sigma$ on the same node set. This parameterization is obviously is not node-wise adaptive, which is critical for our problem. 
We present an example to demonstrate this point. Fig.\ref{fig:2g_graph}(a) shows an unbalanced proximal density, with a ``shallow'' valley in the cut-value curves(red) in (b),(c). RCut curves on ``parameterized'' traditional graphs and our RMD graph are plotted in (b),(c). Note that large values of $k$, $\epsilon$ and $\sigma$ tend to smooth the curve (sometimes even lose the valley) and increase $q$, which worsens the problem. In contrast reducing $k$, $\epsilon$ and $\sigma$ below well-understood thresholds leads to zigzag curves, disconnected graphs and sensitivity to outliers.
Basically, increasing/reducing $k$,$\epsilon$ or $\sigma$ results in uniformly larger/smaller cut-values for all nodes, leading to poor control of $q$.
On the contrary, our rank-modulated degree (RMD) scheme results in fewer/more edges for nodes in low/high density areas, directly reducing the cut-ratio $q$. RCut minima on the RMD graph (black) tends to be near valleys as seen in Fig.\ref{fig:2g_graph}(b),(c). In addition RMD graph also inherits from $k$-NN the advantage of being robust to outliers, as the RCut curve increases near boundaries.
%
\section{RMD Graphs: Main Steps}\label{sec:RMD_idea}
%
Our RMD graph-based learning framework has the following steps:
%
%

\noindent
{\bf (1) Rank Computation:}
The rank $R(x)$ of every point $x$ is calculated:
\begin{eqnarray}\label{eq:grank}
  R(x) = \frac{1}{n}\sum_{i=1}^n\mathbb{I}_{\{G(x)\leq G(x_i)\}}
\end{eqnarray}
where $\mathbb{I}$ denotes the indicator function. Ideally we would like to choose $G(\cdot)$ to be the underlying density, $f(\cdot)$ of the data. Since $f$ is unknown, we need to employ some surrogate statistic. While many choices are possible, the statistic in this paper is based on nearest-neighbor distances. Such statistics have been employed for high-dimensional anomaly detection \citep{Zhao09,Zhao12}. Details are described in Sec.\ref{subsec:rank}. The rank is a normalized ordering of all points based on $G$, ranges in $[0,1]$, and indicates how extreme $x$ is among all points.

\noindent
{\bf (2) Parameterized RMD Graphs Construction:}
Build RMD graphs by connecting each point $x$ to its deg($x$) closest neighbors. The degree deg($x$) for node $x$ is modulated as follows:
\begin{eqnarray}\label{eq:degree}
  deg(x) = k(\lambda+2(1-\lambda)R(x)),
\end{eqnarray}
where $\lambda \in (0,1]$ parameterizes the family of RMD graphs.
$k$ is the average degree. We discretize $\lambda$ in $\{0.2,0.4,0.6,0.8,1\}$ in experiments. It is not difficult to see that $R(x)$ converges (in distribution) to a uniform measure on the unit interval regardless of the underlying density $f(\cdot)$. This implies the average degree across all samples is $k$.
The minimum degree $\lambda k$ can be used to ensure a connected graph when necessary.
Note that we also vary $k$,$\sigma$ in our experiments for a more thorough demonstration.

\noindent
{\bf (3) Graph-based Learning:}
Apply graph-based clustering or SSL algorithms on the family of RMD graphs to obtain a family of partitions. RCut/NCut based SC algorithms are well established. We use both objectives in Sec.\ref{sec:experiment}, but mainly focus on NCut since it has better performance and is recommended. For SSL tasks we employ RCut-based Gaussian Random Fields(GRF) and NCut-based Graph Transduction via Alternating Minimization(GTAM). These approaches all involve minimizing $Tr(F^TLF)$ plus some constraints or penalties, where $L$ is the graph Laplacian, $F$ the cluster indicator or classification function. This is related to RCut(NCut) minimization \citep{Chung96}. We refer readers to references \citep{Zhu08,WanJebCha08,Luxburg07} for details.

\noindent
{\bf (4) Min-Cut Model Selection:}
The final step is to select the min-cut partition that is meaningful according to Eq.(\ref{e.empopt}).
Our main assumption is we have prior knowledge that the smallest cluster is at least of size $\delta n$.
The $K$-partitions obtained from step (3) are now parameterized: $\left(C_1(\lambda,k,\sigma),...,C_K(\lambda,k,\sigma)\right)$.
We pick the partition with minimum Cut value (lowest density valley) over all admissible choices:
\begin{eqnarray}\label{eq:selection}
  & \min_{\lambda,k,\sigma}\{Cut_0\left(C_1,...,C_K\right)=\sum^K_{i=1}Cut_0(C_i,\bar{C}_i)\} \\
\nonumber
  & s.t. ~~\min\{|C_1(\lambda,k,\sigma)|,...,|C_K(\lambda,k,\sigma)|\}\geq \delta n
\end{eqnarray}
Partitions with smaller clusters than $\delta n$ will be discarded. $Cut_0\left(\cdot\right)$ represents the Cut values of different partitions are evaluated on a same reference $k_0$-NN graph to pick the min-cut partition.
{\bf This step exactly aims at the optimal criterion of Eq.(\ref{e.empopt})}.
Note that whatever RCut/NCut is used, for the above size constraint we just consider the number of points within the clusters.


\noindent
\begin{tabular}{lll}
  \hline
  \noindent\textbf{Algorithm 1: RMD Graph-based Learning:} \\
{\bf Input}: $n$ data samples $\{x_1,\ldots,x_n\}$ (partially\\
labeled for SSL), number of clusters/classes $K$, \\
smallest cluster/class size threshold $\delta$. \\
{\bf Steps}:\\
  1. Compute ranks of samples based on Eq.(\ref{eq:grank}). \\
  2. For different $\lambda,k,\sigma$, do: \\
  \indent a. Construct the RMD graph based on Eq.(\ref{eq:degree}); \\
  \indent b. Apply graph-based learning algorithms on the\\
  \indent \indent current RMD graph to get $K$ clusters. \\
  3. Compute Cut values of different partitions from \\
  step 2 on the $k_0$-NN graph. Pick the partition with \\
  the smallest Cut value based on Eq.(\ref{eq:selection}).\\
{\bf Output}: the selected $K$-partition. \\
  \hline
\end{tabular}

\noindent
{\bf Remark:}
Our framework improves the graph construction step, augments with a model selection step with desired optimal criterion, but does not change graph-based learning algorithms. This implies that our framework can be combined with other graph partitioning algorithms to improve performance for unbalanced data, such as ratio/normalized Cheeger cut \citep{Buhler09}.


\subsection{Rank Computation}\label{subsec:rank}
We now specify the statistic $G$ in rank computation. We choose the statistic $G$ in Eq.(\ref{eq:grank}) based on nearest-neighbor distances.
\begin{equation}\label{equ:G(x)}
G(x)=\frac{1}{l}\sum^{2l}_{i=l+1}D_{(i)}(x)
\end{equation}
where $D_{(i)}(x)$ denotes the distance from $x$ to its $i$-th nearest neighbor, and $G$ is the average of $x$'s $(l+1)$-th to $2l$-th nearest neighbor distances. Other choices for $G$ are possible. (1) $G(x)$ is the number of neighbors within an $\epsilon$-ball of $x$ or  (2) $G(x)$ is the distance from $x$ to its $l$-th nearest neighbor. Empirically (and theoretically) we have observed that Eq.(\ref{equ:G(x)}) leads to better performance and robustness. The ranks are relative orderings of points and are quite insensitive to the choice of the neighborhood size parameter $l$. To further reduce variance in rank computation we also employ a U-statistic technique \citep{Korolyuk94} with $B$ times of resampling.


\section{Analysis}\label{sec:thm}
%
Our asymptotic analysis show how graph sparsification leads to control of cut-ratio $q$ introduced in Sec.~2. Detailed proofs can be found in supplementary section. Assume the data set $\{x_1,\ldots,x_n\}$  is drawn i.i.d. from density $f$ in $\mathbb{R}^d$. $f$ has a compact support $C$. Let $G=(V,E)$ be the RMD graph. Given a separating hyperplane $S$, denote $C^+$,$C^-$ as two subsets of $C$ split by $S$, $\eta_d$ the volume of unit ball in $\mathbb{R}^d$.

First we show the asymptotic consistency of the rank $R(y)$ at some point $y$. The limit of $R(y)$, $p(y)$, is the complement of the volume of the level set containing $y$. Note that $p$ exactly follows the shape of $f$, and always ranges in $[0,1]$ no matter how $f$ scales.
\begin{thm}\label{rank-pvalue}
If $f(x)$ satisfies some regularity conditions, then as $n\rightarrow\infty$, we have
\begin{equation}
    R(y)\rightarrow p(y):= \int_{\left\{x:f(x)\leq
f(y)\right\}}f(x)dx.
\end{equation}
\end{thm}
\noindent

{\bf Remark:}\\
(1) The value of $R(x)$ is a direct indicator of whether $x$ lies in high/low density regions(Fig.\ref{fig:rwcont}). \\
(2) $R(x)$ is the integral of pdf asymptotically. It's smooth and uniformly distributed in $[0,1]$. This makes it appropriate to modulate the degrees with control of minimum, maximal and average degree.
\vspace*{-0.15in}
\begin{figure}[h]
\begin{centering}
\begin{minipage}[t]{.38\textwidth}
\includegraphics[width = 1\textwidth]{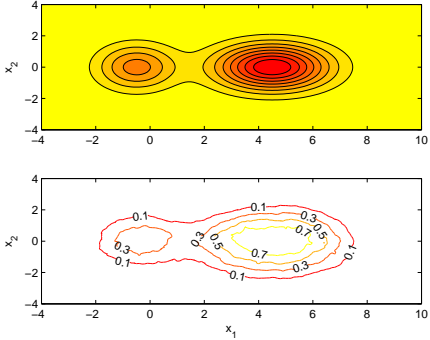}
\end{minipage}
\vspace*{-0.1in}
\caption{\small Density level sets \& rank estimates for unbalanced and proximal gaussian mixtures. High/low ranks correspond to high/low density levels.}
\label{fig:rwcont}
\end{centering}
\end{figure}

%
Next we study RCut(NCut) induced on unweighted RMD graph. The limit cut expression on RMD graph involves an additional adjustable term which varies point-wise according to the density. For technical simplicity, we assume RMD graph ideally connects each point $x$ to its deg$(x)$ closest neighbors.

\begin{thm}\label{part2}
Suppose some smoothness assumptions hold and $S$ be a fixed hyperplane in $\mathbb{R}^d$. For unweighted RMD graph, set the degrees of points according to Eq.(\ref{eq:degree}), where $\lambda\in(0,1)$ is a constant. Let $\rho(x)=\lambda+2(1-\lambda)p(x)$. Assume $k_n/n\rightarrow{0}$. In case $d$=1, assume $k_n/\sqrt{n}\rightarrow\infty$; in case $d\geq$2 assume $k_n/\log{n}\rightarrow\infty$. Then as $n\rightarrow\infty$ we have that:
\begin{equation} \label{eq:rcut}
    \frac{1}{k_n}\sqrt[d]{\frac{n}{k_n}}RCut_n(S)\longrightarrow  C_d B_S \int_S{f^{1-\frac{1}{d}}(s)\rho(s)^{1+\frac{1}{d}}ds}.
\end{equation}
\begin{equation} \label{eq:ncut}
    \sqrt[d]{\frac{n}{k_n}}NCut_n(S)\longrightarrow  C_d B_S \int_S{f^{1-\frac{1}{d}}(s)\rho(s)^{1+\frac{1}{d}}ds}.
\end{equation}
where $C_d = \frac{2\eta_{d-1}}{(d+1)\eta_d^{1+1/d}}$, $B_S=\left(\mu(C^+)^{-1}+\mu(C^-)^{-1}\right)$,  and $\mu(C^{\pm})=\int_{C^{\pm}}f(x)dx$.
\end{thm}
%

{\bf Remark:} \\
(1) Compared to the limit expression on $k$-NN graph, there is an additional term $\rho(s)=(\lambda+2(1-\lambda)p(s))$. The monotonicity of $\rho(s)$ in $p(s)$ immediately implies that the ``infinitesimal'' cut contribution at low(high) density areas is reduced (increased).  To see the impact suppose $\lambda$ is small; we see that for cuts $S$ near modes, $p(s)\approx 1$ and this extra term is nearly $(2)^{1+\frac{1}{d}}$. For $S$ near valleys this term is nearly $(\lambda)^{1+\frac{1}{d}}<1$.
The cut-ratio $q$ is explicitly reduced. \\
(2) Smaller $\lambda$ further penalizes high density areas over low density areas, further reduces the cut-ratio $q$ and pulls down $(q,y)$ in Fig.1, thus has the ability to cope with even more unbalanced data (with smaller $y$). Therefore, without a priori information about how unbalanced the data is, parameterizing graphs with varying values of $\lambda$ provides for RCut(NCut) based algorithms the ability to adapt to data with varying levels of unbalancedness.

%
\section{Simulations}\label{sec:experiment}



Experiments in this section involve both synthetic and real data sets.
We focus on unbalanced data by randomly sampling from different classes in an unbalanced manner.
As for traditional graphs we also include $b$-matching graph \citep{JebWanCha09} with $b=k$.

For clustering experiments we apply both RCut and NCut based SC, but focus on NCut since it is generally known to perform better. We report performance by evaluating how well the clusters structures match the ground truth labels, as is the standard criterion for partitional clustering \citep{xu05}. For instance consider Tab.1 where error rates for USPS symbols 1,8,3,9 are tabulated. We parameterize various graphs and apply SC to get various partitions. Our model selection scheme picks the partition according to Eq.(\ref{eq:selection}), AGNOSTIC to the correspondence between samples and symbols. Errors are then reported by looking at mis-associations.

For SSL experiments we randomly pick labeled points among unbalancedly sampled data, guaranteeing at least one labeled from each class. SSL algorithms such as RCut-based GRF and NCut-based GTAM are applied on parameterized graphs built from partially labeled data, and generate various partitions. The model selection scheme picks the min-cut partition simply based on graph structures according to Eq.(\ref{eq:selection}). Then labels for unlabeled data are predicted based on the selected parition and compared against the UNKNOWN true labels to produce the error rates

Some general simulation parameters are:\\
{\bf (1)} We employ U-statistic technique in rank computation to reduce variance (Sec.\ref{subsec:rank}), with $B=5$.\\
{\bf (2)} All error rate results are averaged over 20 trials.\\
Other parameters will be specified below.



\subsection{Synthetic DataSets}\label{subsec:syn}
\begin{figure*}[tb]
\begin{centering}
\begin{minipage}[t]{.24\textwidth}
\includegraphics[width = 1\textwidth]{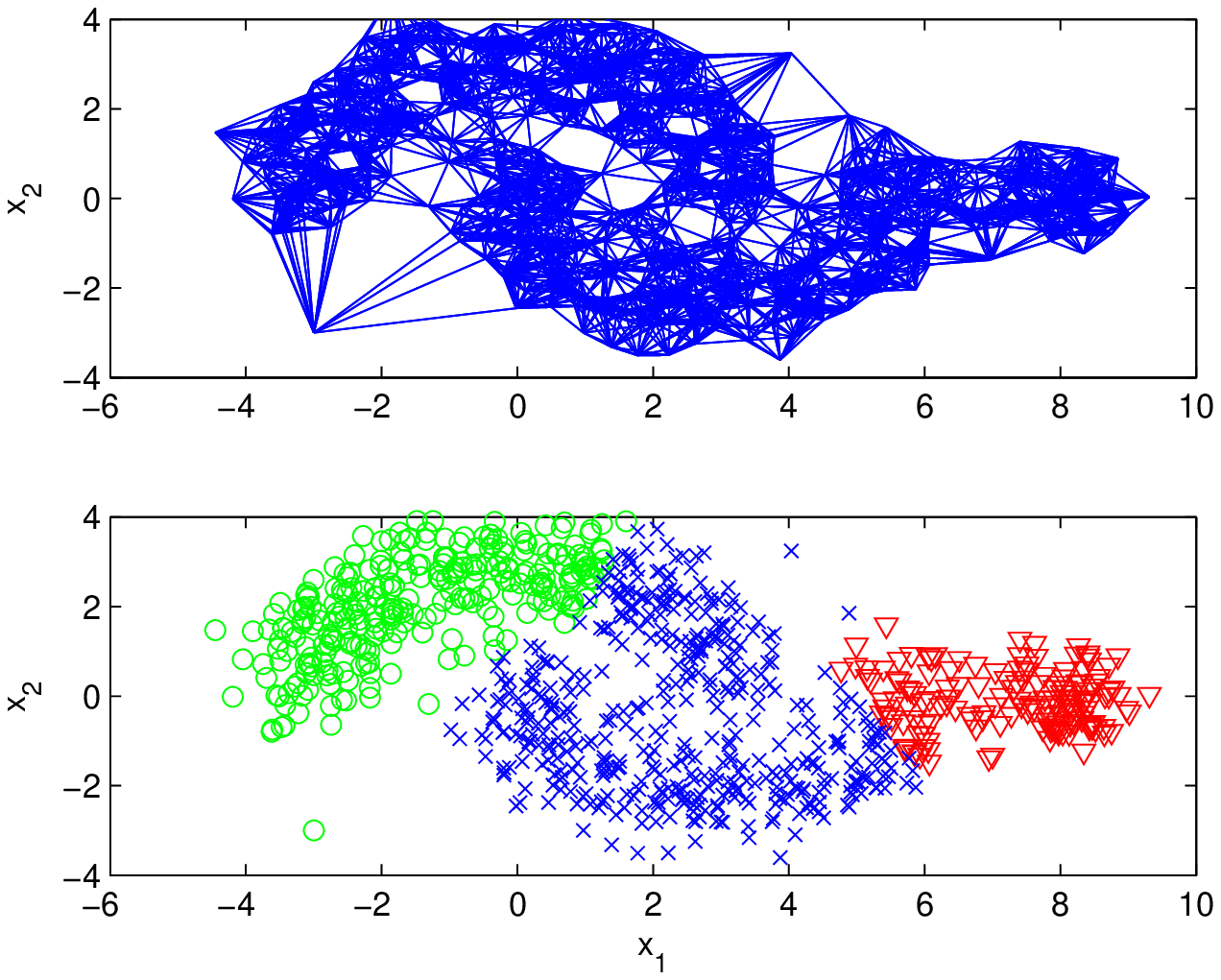}
\makebox[4cm]{\small (a) $k$-NN}
\end{minipage}
\begin{minipage}[t]{.24\textwidth}
\includegraphics[width = 1\textwidth]{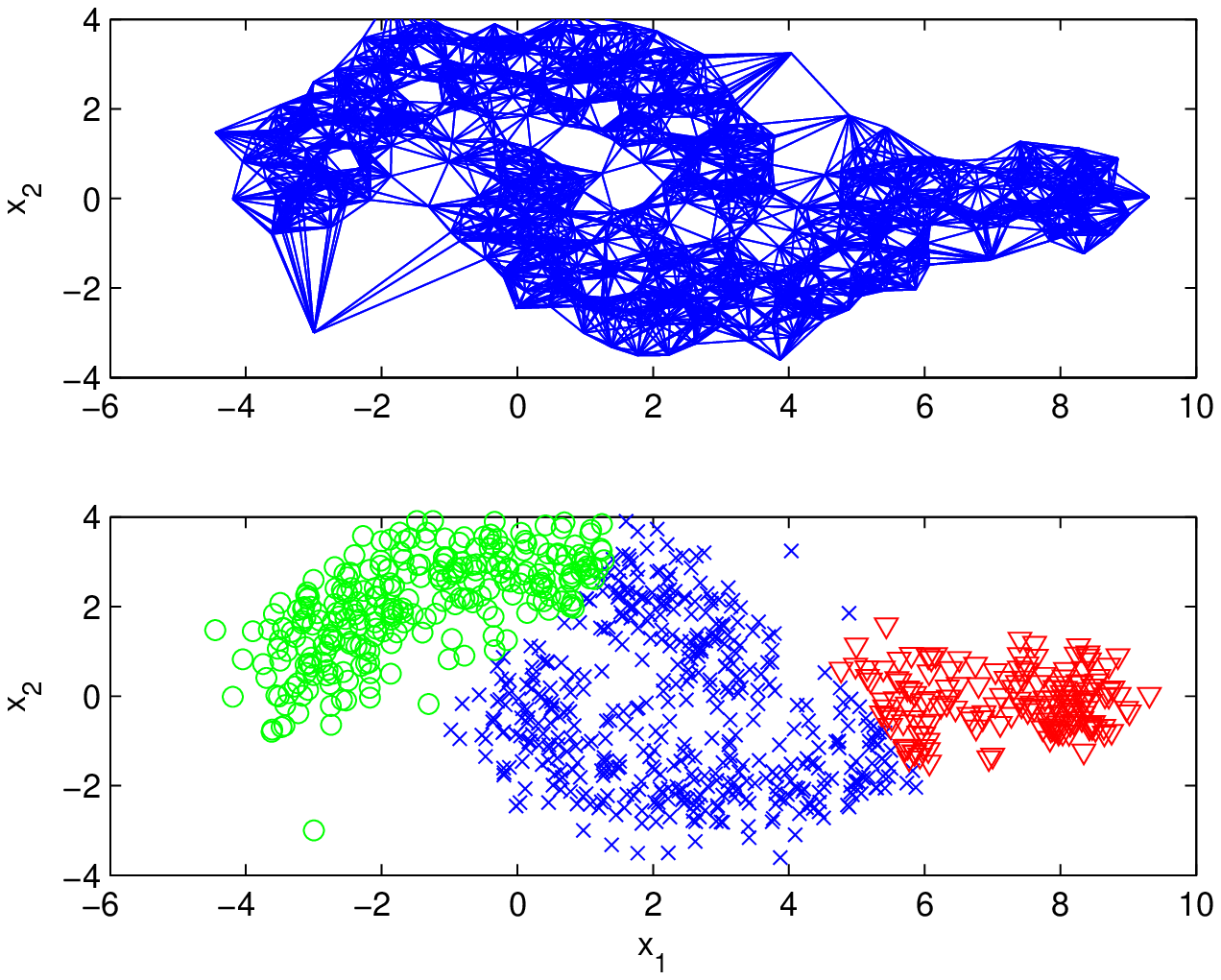}
\makebox[4cm]{\small (b) $b$-matching}
\end{minipage}
\begin{minipage}[t]{.24\textwidth}
\includegraphics[width = 1\textwidth]{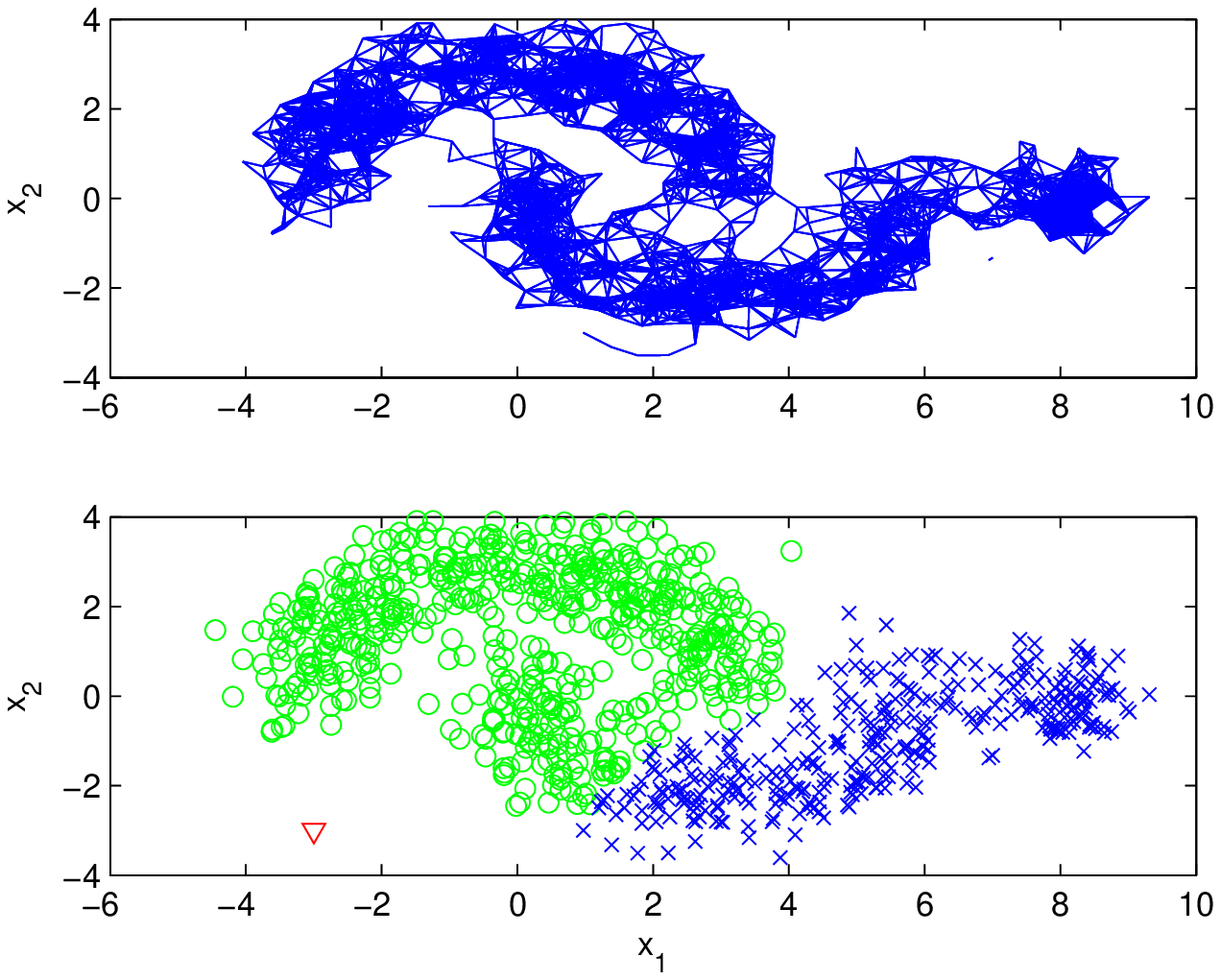}
\makebox[4cm]{\small (c) $\epsilon$-graph(full-RBF)}
\end{minipage}
\begin{minipage}[t]{.24\textwidth}
\includegraphics[width = 1\textwidth]{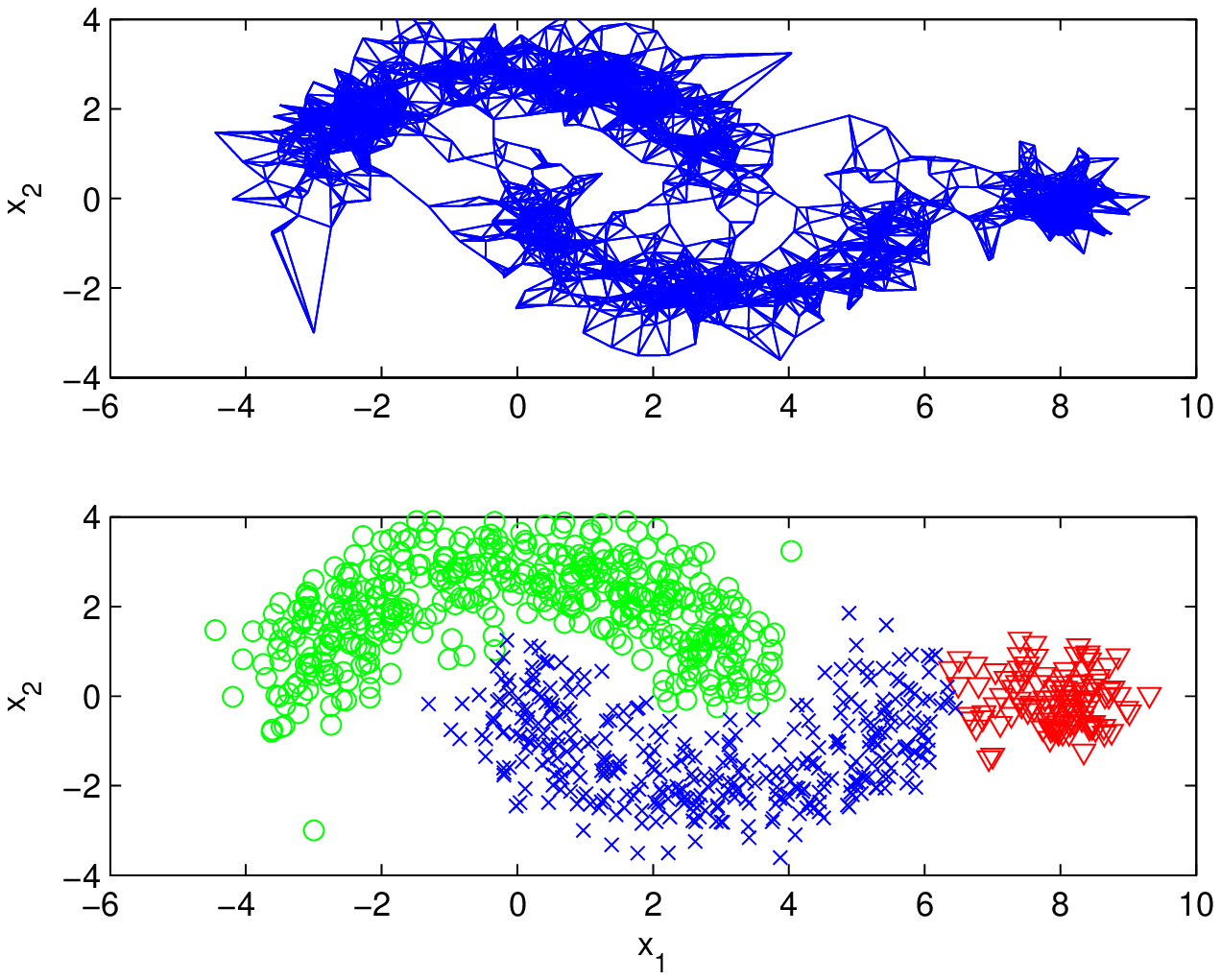}
\makebox[4cm]{\small (d) RMD}
\end{minipage}
\caption{\small Clustering results of 3-partition SC on 2 moons and 1 gaussian data set. SC on full-RBF($\epsilon$-graph) completely fails due to the outlier. For $k$-NN and $b$-matching graphs SC cannot recognize the long winding low-density regions between 2 moons, and fails to find the rightmost small cluster. Our method sparsifies the graph at low-density regions, allowing to cut along the valley, detect the small cluster and is robust to outliers.}
\label{fig:complex_shape}
\end{centering}
\end{figure*}
Consider a multi-cluster complex-shaped data set, which is composed of 1 small Gaussian and 2 moon-shaped proximal clusters shown in Fig.\ref{fig:complex_shape}. Sample size $n=1000$ with the rightmost small cluster $10\%$ and two moons $45\%$ each. In this example, for illustration we did not parameterize the graph or apply the model selection step. We fix $\lambda = 0.5$, and choose $k=l=30$, $\epsilon=\sigma=\tilde{d}_k$, where $\tilde{d}_k$ is the average $k$-NN distance. Model-based approaches can fail on such dataset due to the complex shapes of clusters. The 3-partition SC based on RCut is applied. On $k$-NN and $b$-matching graphs SC fails for two reasons: (1) SC cuts at balanced positions and cannot detect the rightmost small cluster; (2) SC cannot recognize the long winding low-density regions between 2 moons because there are too many spurious edges and the Cut value along the curve is big. SC fails on $\epsilon$-graph(similar on full-RBF) because the outlier point forms a singleton cluster, and also cannot recognize the low-density curve. RMD graph significantly sparsifies the graph at low-density regions, enabling SC to cut along the valley, detect the small cluster and is robust to outliers.
\subsection{Real DataSets}\label{subsec:real}
We focus on unbalanced settings and consider several real datasets. We construct $k$-NN, $b$-match, full-RBF and RMD graphs all combined with RBF weights, but do not include the $\epsilon$-graph because of its overall poor performance \citep{JebWanCha09}.
We discretize not only $\lambda$ but also $k$, $\sigma$ to parameterize graphs.
The sample size is around 750 to 1500, described respectively. We vary $k$ in $\{10,20,30,...,100\}$.
Note that although small $k$ in our scheme may lead to disconnected graphs due to minimum degree $\lambda k$ in Eq.(\ref{eq:degree}), the resulting partitions with singleton clusters will be ruled out by the constraints of Eq.(\ref{eq:selection}).
Also notice that for $\lambda=1$, RMD graph is identical to $k$-NN graph.
For RBF parameter $\sigma$ it has been suggested to be of the same scale as the average $k$-NN distance $\tilde{d}_k$ \citep{WanJebCha08}. This suggests a discretization of $\sigma$ as $2^j \tilde{d}_k$ with $j=-3,\,-2,\ldots,\,3$.
We discretize $\lambda \in \{0.2,0.4,0.6,0.8,1\}$.
In the model selection step Eq.(\ref{eq:selection}), cut values of various partitions are evaluated on a same $k_0$-NN graph with $k_0=30, \sigma = \tilde{d}_{30}$ before selecting the min-cut partition.
$l$ is fixed to be 30. The true number of clusters/classes $K$ is supposed to be known. We assume meaningful clusters are at least $5\%$ of the total number of points, $\delta=0.05$. We set the GTAM parameter $\mu=0.05$ \citep{JebWanCha09} for the SSL tasks, and each time 20 randomly labeled samples are chosen with at least one sample from each class.

\begin{figure}[tb]
\begin{centering}
\begin{minipage}[t]{.23\textwidth}
\includegraphics[width = 1\textwidth]{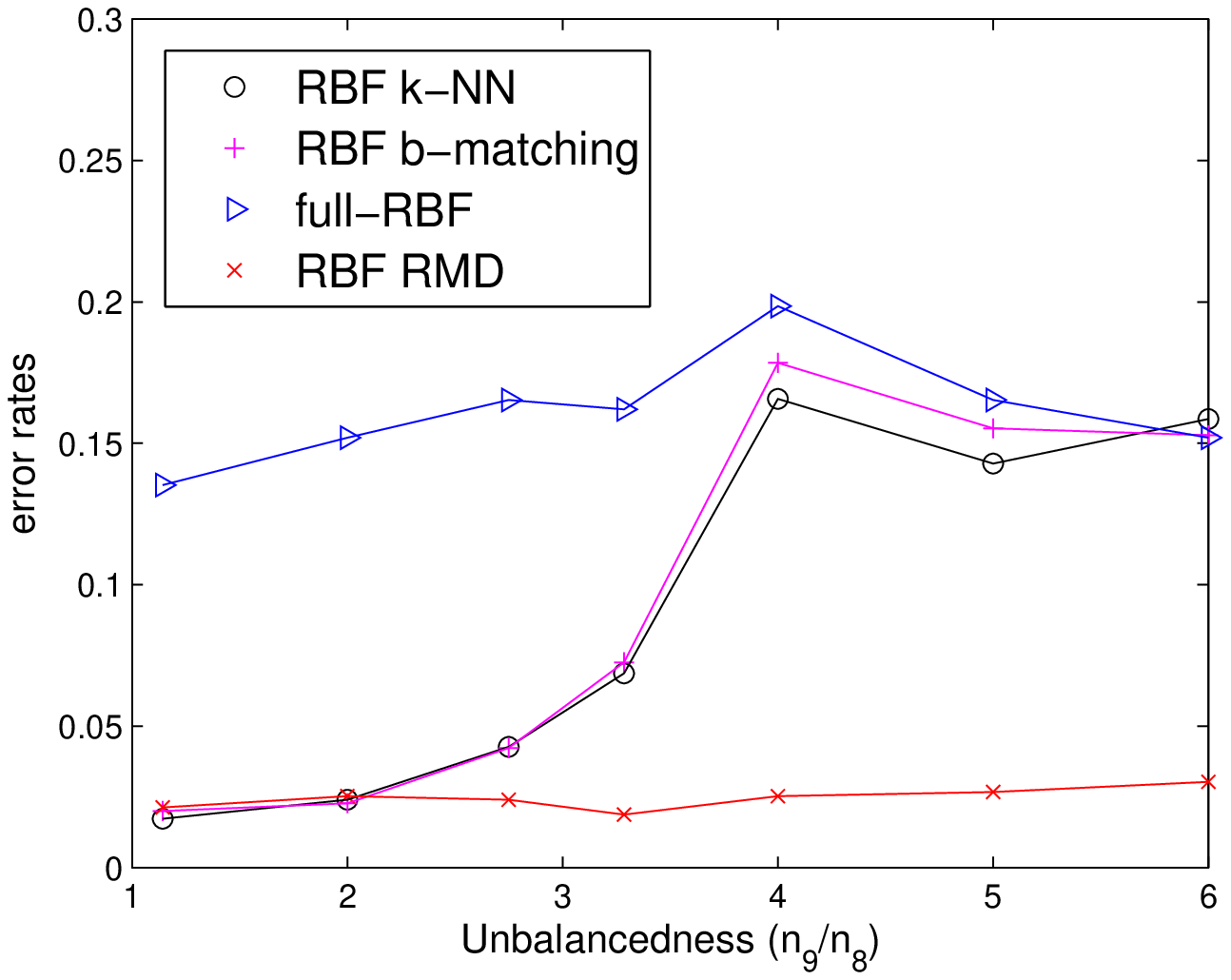}
\makebox[3cm]{\small (a) SC}
\end{minipage}
\begin{minipage}[t]{.23\textwidth}
\includegraphics[width = 1\textwidth]{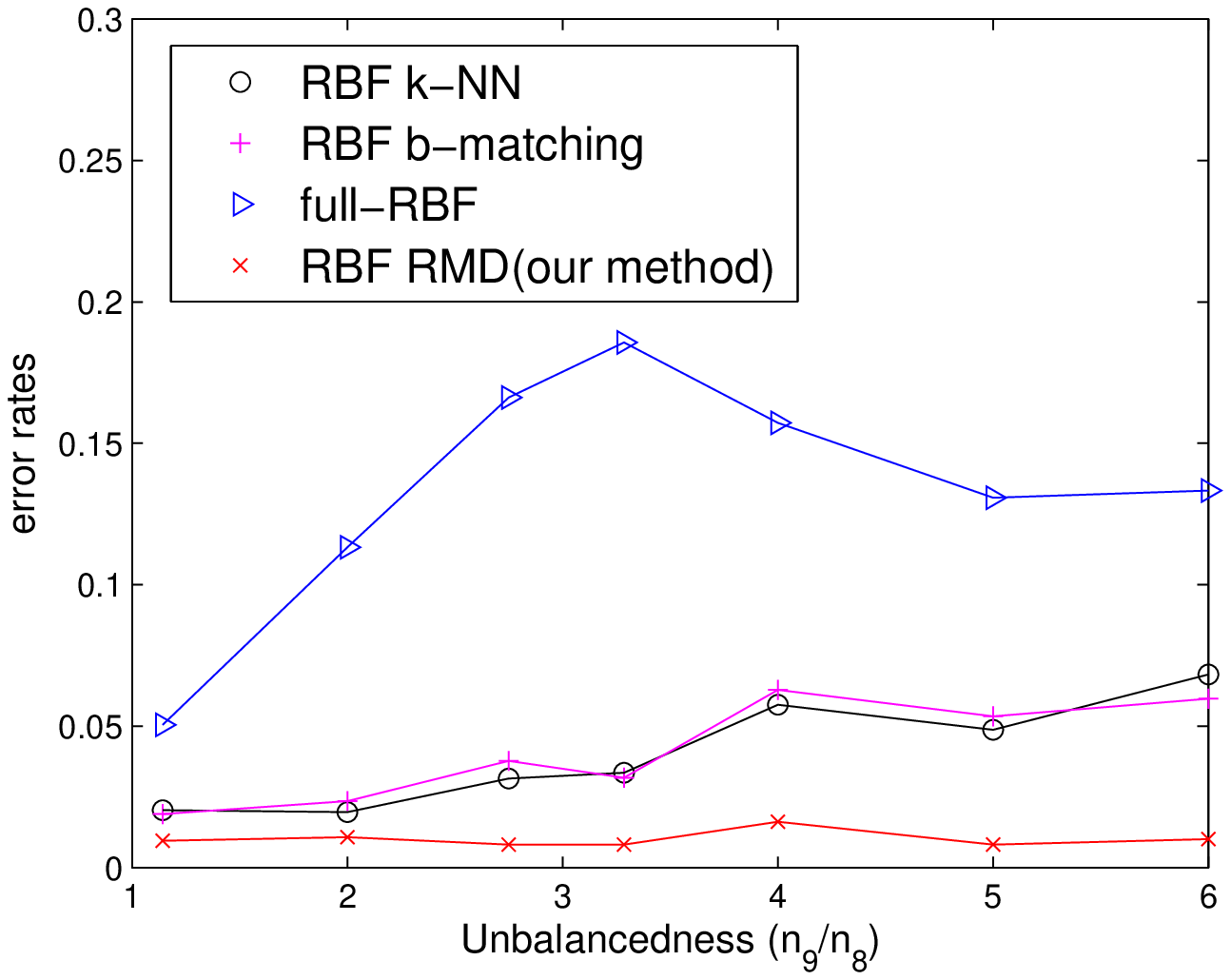}
\makebox[3cm]{\small (b) GTAM}
\end{minipage}
\caption{\small Error rate performance of SC and GTAM on USPS 8vs9 with varying levels of unbalancedness. We omitted GRF since the results are qualitatively similar. Our method adapts to different levels of unbalancedness much better than traditional graphs. Furthermore, when data is very unbalanced (big $n_9/n_8$), varying $k,\sigma$ does not really help; decreasing $\lambda$ adapts the algorithm well.}
\label{fig:USPS8v9}
\end{centering}
\vspace{-0.2in}
\end{figure}
\textbf{Varying Unbalancedness:}
We start with a comparison for 8vs9 of the 256-dim USPS digit data set. We keep the total sample size as 750, and vary the unbalancedness, i.e. the proportion of numbers of points randomly sampled from two classes, denoted by $n_8,n_9$. Normalized SC and GTAM are applied. Fig.\ref{fig:USPS8v9} shows that when the underlying clusters/classes are balanced, our method works as perfect as traditional graphs; as the unbalancedness increases, the performance severely degrades on traditional graphs, while our method can adapt the graph-based learning algorithms to different levels of unbalancedness very well.

\begin{table*}[tb]
\caption{\small Error rate performance of normalized SC on various graphs for unbalanced real data sets. Our method performs significantly better than other methods.}
\begin{center}
\begin{tabular}{|c||c|c|c|c|c|c|c|c|c|c|}
  \hline
  \multirow{2}{*}{Error Rates(\%)}   &   \multicolumn{2}{c|}{USPS}  &   \multicolumn{3}{c|}{SatImg}  &   \multicolumn{3}{c|}{OptDigit}   & \multicolumn{2}{c|}{LetterRec} \\
  \cline{2-11}
  & 8vs9 & 1,8,3,9 & 4vs3 & 3,4,5 & 1,4,7 & 9vs8 & 6vs8 & 1,4,8,9 & 6vs7 & 6,7,8 \\
  \hline\hline
  RBF $k$-NN        & 16.67 & 13.21 & 12.80 & 18.94 & 25.33 & 9.67  & 10.76  & 26.76 & 4.89 & 37.72 \\
  RBF $b$-matching  & 17.33 & 12.75 & 12.73 & 18.86 & 25.67 & 10.11  & 11.44  & 28.53 & 5.13 & 38.33 \\
  full-RBF          & 19.87 & 16.56 & 18.59 & 21.33 & 34.69 & 11.61 & 15.47 & 36.22 & 7.45 & 35.98 \\
  full-aRBF         & 18.35 & 16.26 & 16.79 & 20.15 & 35.91 & 10.88 & 13.27 & 33.86 & 7.58 & 35.27 \\
  RBF RMD           & 4.80  & 9.18 & 7.87 & 15.26 & 19.72 & 5.43  & 6.67  & 21.35 & 2.92 & 28.68 \\
  \hline
\end{tabular}
\end{center}
\label{tab:real_SC}
\end{table*}
\begin{table*}[tb]
\caption{\small Error rate performance of GRF and GTAM for unbalanced real data sets. Our method performs significantly better than other methods.}
\begin{center}
\begin{tabular}{|c|c||c|c|c|c|c|c|c|c|c|}
  \hline
  \multicolumn{2}{|c||}{\multirow{2}{*}{Error Rates(\%)}}  &   \multicolumn{2}{c|}{USPS}  &   \multicolumn{2}{c|}{SatImg}  &   \multicolumn{3}{c|}{OptDigit}   & \multicolumn{2}{c|}{LetterRec} \\
  \cline{3-11}
  \multicolumn{2}{|c||}{}  & 8vs6 & 1,8,3,9 & 4vs3 & 1,4,7 & 6vs8 & 8vs9 & 6,1,8 & 6vs7 & 6,7,8 \\
  \hline\hline
  \multirow{4}{*}{GRF}
    & RBF $k$-NN            & 5.70 & 13.29 & 14.64 & 16.68 & 5.68  & 7.57  & 7.53 & 7.67 & 28.33 \\
    & RBF $b$-matching      & 6.02 & 13.06 & 13.89 & 16.22 & 5.95  & 7.85  & 7.92 & 7.82 & 29.21 \\
    & full-RBF              & 15.41 & 12.37 & 14.22 & 17.58 & 5.62 & 9.28 & 7.74 & 11.52 & 28.91 \\
    & full-aRBF             & 12.89 & 11.74 & 13.58 & 17.86 & 5.78 & 8.66 & 7.88 & 10.10 & 28.36 \\
    & RBF RMD               & 1.08  & 10.24 & 9.74 & 15.04 & 2.07  & 2.30  & 5.82 & 5.23 & 27.24 \\
  \hline
  \multirow{4}{*}{GTAM}
    & RBF $k$-NN            & 4.11  & 10.88 & 26.63 & 20.68 & 11.76 & 5.74  & 12.68 & 19.45 & 27.66 \\
    & RBF $b$-matching      & 3.96  & 10.83 & 27.03 & 20.83 & 12.48 & 5.65  & 12.28 & 18.85 & 28.01 \\
    & full-RBF              & 16.98  & 11.28 & 18.82 & 21.16 & 13.59 & 7.73 & 13.09 & 18.66 & 30.28 \\
    & full-aRBF             & 13.66  & 10.05 & 17.63 & 22.69 & 12.15 & 7.44 & 13.09 & 17.85 & 31.71 \\
    & RBF RMD               & 1.22  & 9.13 & 18.68 & 19.24 & 5.81  & 3.12  & 10.73 & 15.67 & 25.19 \\
  \hline
\end{tabular}
\end{center}
\label{tab:real_SSL}
\end{table*}
\textbf{Other Real Data Sets:}
We apply SC and SSL algorithms on several other real data sets including USPS, waveform database generator(21-dim), Statlog landsat satellite images(36-dim), letter recognition images(16-dim) and optical recognition of handwritten digits(64-dim) \citep{uci10}.
We randomly sample 150/600, 200/400/600, 200/300/400/500 points for 2,3,4-class cases, with corresponding orders of class indices listed in Tab.\ref{tab:real_SC},\ref{tab:real_SSL}.
For comparison we also include the full graph with adaptive RBF weights (full-aRBF), where $\sigma_u$ is chosen as the $k$-NN distance of node $u$, and $w(u,v)=exp\left(-d(u,v)^2/2\sigma_u\sigma_v\right)$ \citep{Zelnik04}.
Tab.\ref{tab:real_SC},\ref{tab:real_SSL} shows that varying $k$, $\sigma$ for traditional graphs does not work well, while our method consistently performs better.

\subsection{Applications to Small Cluster Detection}
We illustrate how our method can be used to find small-size clusters. This type of problem may arise in community detection in large real networks, where graph-based approaches are popular but small-size community detection is difficult \citep{Shah10}.

The dataset depicted in Fig.\ref{fig:multiple_cuts} has 1 large and 2 small proximal Gaussian components along $x_1$ axis: $\sum^{3}_{i=1}\alpha_iN(\mu_i,\Sigma_i)$, where $\alpha_1:\alpha_2:\alpha_3=2:8:1$, $\mu_1$=[-0.7;0], $\mu_2$=[4.5;0], $\mu_3$=[9.7;0], $\Sigma_1=I, \Sigma_2=diag(2,1), \Sigma_3=0.7I$. Binary weight is adopted.

Fig.\ref{fig:multiple_cuts}(a) shows a plot of cut values for different cut positions averaged over 20 Monte Carlo runs. We note that the cut-value plot resembles the underlying density. Two density valleys are both at unbalanced positions. The rightmost cluster is smaller than the left cluster, but has a deeper valley.

To apply our method we vary the cluster-size threshold $\delta$ in Eq.(\ref{eq:selection}). We now plot the Cut-value against $\delta$ as shown in Fig.\ref{fig:multiple_cuts}(b). As seen in Fig.\ref{fig:multiple_cuts}(b), when $\delta\geq 0.3$, the optimal cut is close to the valley. However, since the proportion of data samples in the smaller clusters is less than 30\% we see that the optimal cut is bounded away from both valleys. As $\delta$ is decreased in the range $0.25\geq\delta\geq0.15$, the optimal cut is now attained at the left valley($x_1\approx 1.8$). An interesting phenomena is that the curve flattens out in this range.
This corresponds to the fact that the cut value is minimized at this position ($x_1 = 1.8$) for any value of $\delta \in [.15,\,.25]$. This flattening out can happen only at valleys since valleys represent a ``local'' minima for the model selection step of Eq.~\ref{eq:selection} under the constraint imposed by $\delta$. Consequently, small clusters can be detected based on the flat spots. Next when we further vary $\delta$ in the region $0.1\geq\delta\geq0.05$, the best cut is attained near the right and deeper valley($x_1\approx 8.2$). Again the curve flattens out revealing another small cluster.
\begin{figure*}[tb]
\begin{centering}
\begin{minipage}[t]{.32\textwidth}
\includegraphics[width = 1\textwidth]{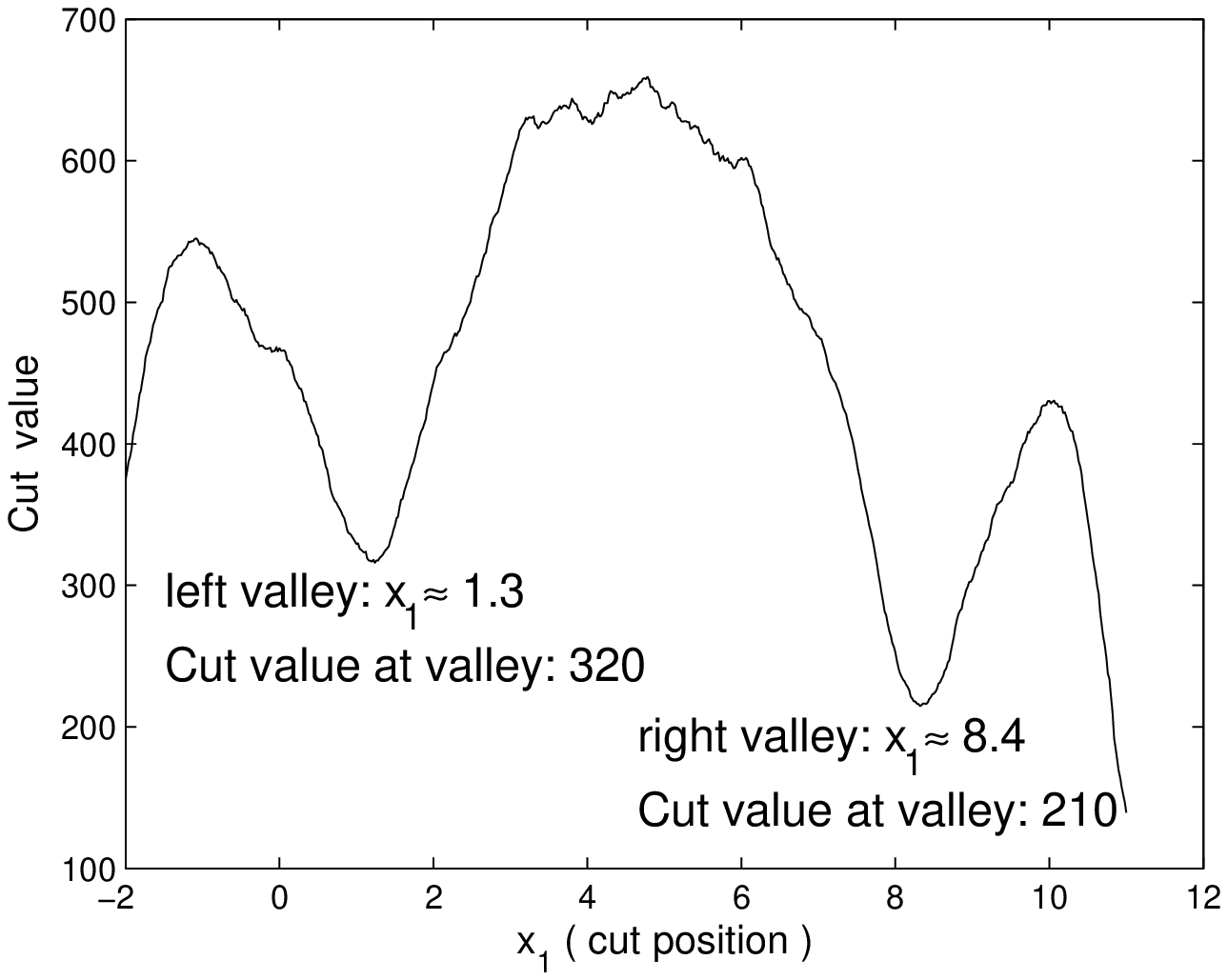}
\makebox[5.5cm]{\small (a) Cut value vs. cut position}
\end{minipage}
\begin{minipage}[t]{.32\textwidth}
\includegraphics[width = 1\textwidth]{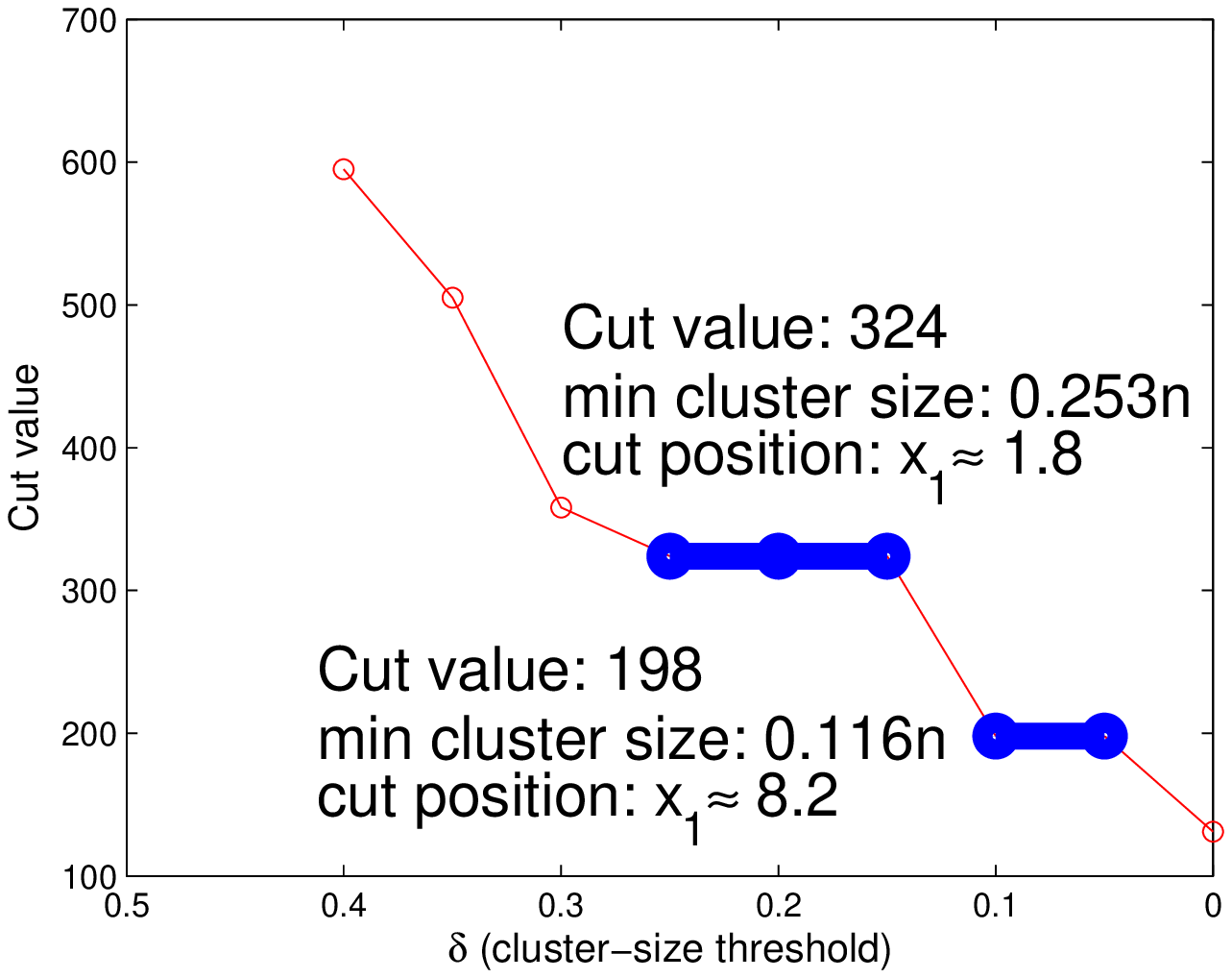}
\makebox[5.5cm]{\small (b) Cut value vs. Cluster size($\delta$)}
\end{minipage}
\begin{minipage}[t]{.32\textwidth}
\includegraphics[width = 1\textwidth]{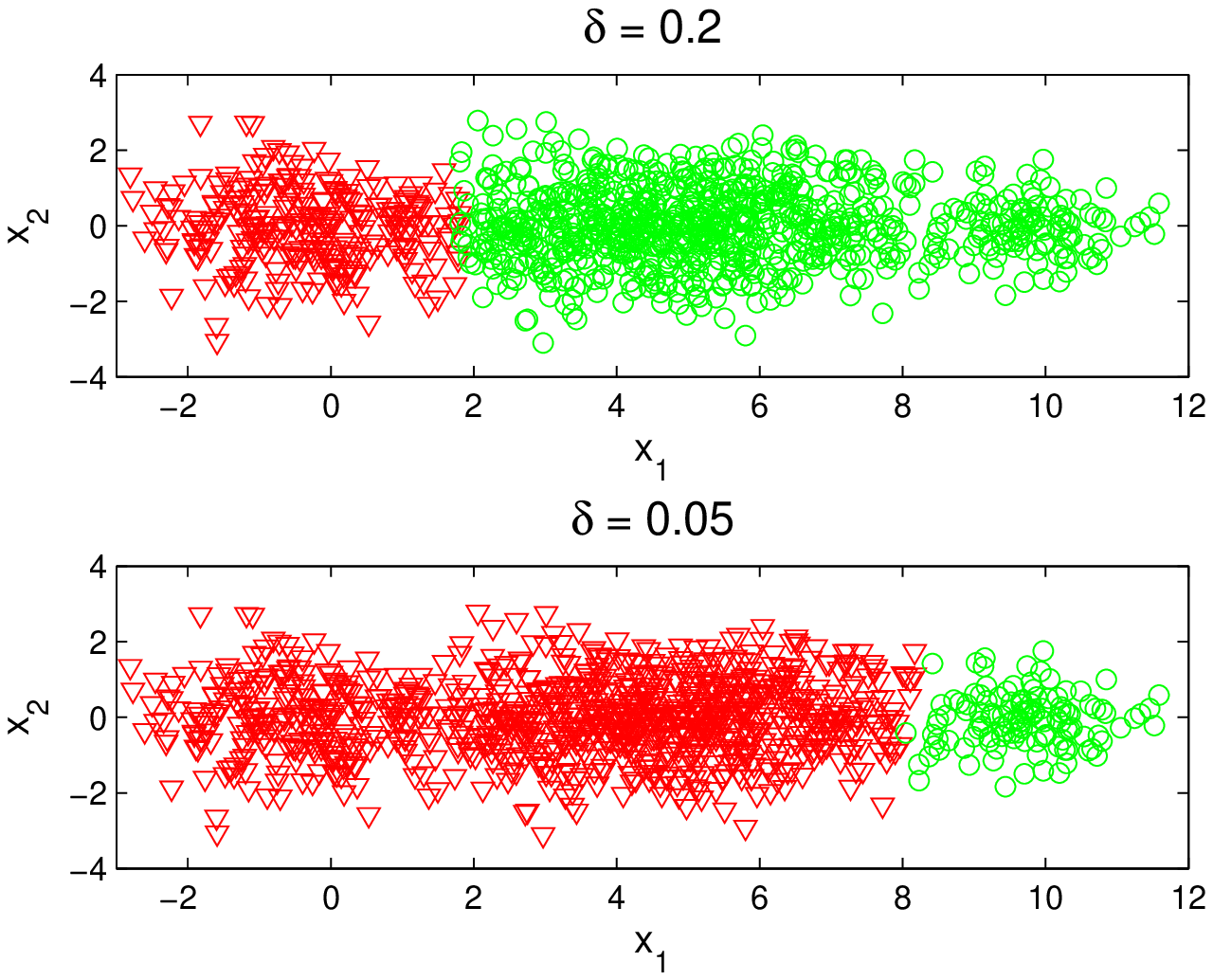}
\makebox[5.5cm]{\small (c) different clustering results }
\end{minipage}
\caption{\small 2-partition SC results of 1 large and 2 small proximal gaussian mixture components. Both valleys are at unbalanced positions. The rightmost cluster is smaller than the left, with a deeper valley. Results in (b) are from one run. As shown in (b) and (c), the left cluster is detected for a larger $\delta$, where the right smaller one is viewed as outliers. When even reducing $\delta$, the right smaller one is detected(Eq.(\ref{eq:selection})).}
\label{fig:multiple_cuts}
\end{centering}
\end{figure*}

\subsection{Comments}\label{subsec:discRMD}
{\bf Tuning Parameters:}
$\lambda$ is a parameter that is optimized through the model selection step and does not count as a tuning parameter(so are $k$ and $\sigma$ under our framework). The choice of $\delta$ is based on our prior to find sizable clusters, say 5\% to 10\% of the data.
As for $k_0$ and $l$, our method appears to be relatively insensitive to the values of $k_0,l$. Unlike graph parameters $\lambda,k,\sigma$ which have direct impact on graph-based algorithms, $k_0$ is used to relatively compare different partitions and $l$ is used to relatively order data points.
It is not surprising that the relative ranking of high/low density cuts (or points near high/low density areas) does not substantially change when compared on a nearest neighbor graph with different $k_0$ ($l$), as is usually the case in our experiments.
Similar phenomena have been observed in the context of high-dimensional anomaly detection \citep{Zhao09,Zhao12}.
We fix $k_0=l$ roughly the same scale as $\sqrt n$ in experiments.


\noindent{\bf Time Complexity:}
The time complexity of U-statistic rank computation is $O(Bdn^2logn)$, where $B$ is a small constant, 5 in our experiments. RMD graph construction is $O(dn^2logn)$, same as constructing a $k$-NN graph.
Computing Cut value and checking the sizable cluster constraint for a partition takes $O(n^2)$. So if totally $D$ graphs are parameterized and the graph-based learning algorithm needs $T$, the whole complexity is $O((B+D)dn^2logn+DT)$.

\section{Conclusions}\label{sec:conclusion}
We have shown that RCut(NCut) based spectral methods on traditional graphs can lead to balanced cuts rather than density valley cuts for unbalanced proximal data. We propose a systematic procedure to parameterize graphs based on a rank-modulated degree (RMD) scheme, which adaptively sparsifies/densifies the neighborhoods of nodes. This scheme effectively adapts RCut(NCut) based methods to unbalanced data. We then present a model selection step which allows for best sizable clusters separated by smallest cut value. By constraining the smallest cluster sizes we can detect multiple small clusters and generate different meaningful cuts. Our synthetic and real simulations demonstrate significant performance improvements over existing methods for unbalanced data. The ability to detect small-size clusters indicates our idea may be utilized in other applications such as community detection in large networks, where graph-based approaches are popular but small-size community detection is difficult.

\bibliographystyle{icml2013}
\bibliography{RMD_bib}

\onecolumn

\newpage

\section*{Appendix: Proofs of Theorems}
For ease of development, let $n=m_1(m_2+1)$, and divide $n$ data points into: $D=D_0 \bigcup  D_1 \bigcup ... \bigcup D_{m_1}$, where $D_0=\{x_1,...,x_{m_1}\}$, and each $D_j, j=1,...,m_1$ involves $m_2$ points. $D_j$ is used to generate the statistic $G$ for $u$ and $x_j\in D_0$, for $j=1,...,m_1$. $D_0$ is used to compute the rank of $u$:
\begin{equation}
    R(u) = \frac{1}{m_1}\sum_{j=1}^{m_1} \mathbb{I}_{\{ G(x_j;D_j)>G(u;D_j) \}}
\end{equation}
We provide the proof for the statistic $G(u)$ of the following form:
\begin{eqnarray}
  G(u;D_j) &=& \frac{1}{l}\sum^{l+\lfloor \frac{l}{2} \rfloor}_{i=l-\lfloor \frac{l-1}{2} \rfloor}\left( \frac{l}{i} \right)^{\frac{1}{d}}D_{(i)}(u).
\end{eqnarray}
where $D_{(i)}(u)$ denotes the distance from $u$ to its $i$-th nearest neighbor among $m_2$ points in $D_j$. Practically we can omit the weight as Eq.(\ref{equ:G(x)}) in the paper.

\textbf{Regularity conditions:} $f(\cdot)$ is continuous and lower-bounded: $f(x) \geq f_{min}>0$. It is smooth, i.e. $||\nabla f(x)||\leq\lambda$, where $\nabla f(x)$ is the gradient of $f(\cdot)$ at $x$. Flat regions are disallowed, i.e. $\forall x \in \mathbb{X}$, $\forall \sigma>0$, $\mathcal{P}\left\{y: |f(y)-f(x)|<\sigma\right\}\leq M\sigma$, where $M$ is a constant.

\noindent\textbf{Proof of Theorem 1:}
\begin{proof}
The proof involves two steps:
\begin{itemize}
  \item[1.] The expectation of the empirical rank $\mathbb{E}\left[R(u)\right]$ is shown to converge to $p(u)$ as $n\rightarrow\infty$.
  \item[2.] The empirical rank $R(u)$ is shown to concentrate at its expectation as $n\rightarrow\infty$.
\end{itemize}
The first step is shown through Lemma \ref{lem:expectation}. For the second step, notice that the rank $R(u) = \frac{1}{m_1}\sum_{j=1}^{m_1} Y_j$, where $Y_j = \mathbb{I}_{\{ G(x_j;D_j)>G(u;D_j) \}}$ is independent across different $j$'s, and $Y_j \in [0,1]$. By Hoeffding's inequality, we have:
\begin{equation}
    \mathbb{P}\left( | R(u) - \mathbb{E}\left[R(u)\right] | > \epsilon \right) < 2\exp\left( -2m_1\epsilon^2 \right)
\end{equation}
Combining these two steps finishes the proof.
\end{proof}

\noindent\textbf{Proof of Theorem 2:}
\begin{proof}
We only present a brief outline of the proof. We want to establish the convergence result of the cut term and the balancing terms respectively, that is:
\begin{eqnarray}
    &\frac{1}{nk_n}\sqrt[d]{\frac{n}{k_n}}cut_n(S)
    \rightarrow C_d\int_S{f^{1-\frac{1}{d}}(s)\rho(s)^{1+\frac{1}{d}}ds}. \label{eq:term1}\\
    &n\frac{1}{|V^\pm|}\rightarrow
    \frac{1}{\mu(C^\pm)}. \label{eq:term2R}  \\
    &nk_n\frac{1}{vol(V^\pm)}\rightarrow
    \frac{1}{\mu(C^\pm)}. \label{eq:term2N}
\end{eqnarray}
where $V^+(V^-)=\{x\in{V}: x\in{C^+}(C^-)\}$ are the discrete version of $C^+(C^-)$.

The balancing terms Eq.(\ref{eq:term2R},\ref{eq:term2N}) are obtained similarly using Chernoff bound on the sum of binomial random variables, since the number of points in $V^\pm$ is binomially distributed $Binom(n,\mu(C^\pm))$. Details can be found in \citet{Maier1}.

Eq.(\ref{eq:term1}) is established in two steps. First we can show that the LHS cut term converges to its expectation $\mathbb{E}\left(\frac{1}{nk_n}\sqrt[d]{\frac{n}{k_n}}cut_n(S)\right)$ by McDiarmid's inequality. This can also be found in \citet{Maier1}. Second we show this expectation term actually converges to the RHS of Eq.(\ref{eq:term1}). This is shown in Lemma~\ref{expectation}.

\end{proof}

\begin{lem}\label{expectation}
Given the assumptions of Theorem 2,
\begin{equation}
    \mathbb{E}\left(\frac{1}{nk_n}\sqrt[d]{\frac{n}{k_n}}cut_n(S)\right)\longrightarrow C_d\int_S{f^{1-\frac{1}{d}}(s)\rho(s)^{1+\frac{1}{d}}ds}.
\end{equation}
where $C_d=\frac{2\eta_{d-1}}{(d+1)\eta_d^{1+1/d}}$.
\end{lem}

\begin{proof}
The proof is a simple extension of \citet{Maier2}. We provide an outline here. The first trick is to define a cut function for a fixed point $x_i\in V^+$, whose expectation is easier to compute:
\begin{eqnarray}
cut_{x_i} = \sum_{v\in V^{-},(x_i,v)\in E}w(x_i,v).
\end{eqnarray}
Similarly, we can define $cut_{x_i}$ for $x_i\in V^-$. The expectation of $cut_{x_i}$ and  $cut_n(S)$ can be related:
\begin{eqnarray}\label{eq:expect}
\mathbb{E}(cut_n(S))=n\mathbb{E}_x(\mathbb{E}(cut_{x}))
\end{eqnarray}
Then the value of $\mathbb{E}(cut_{x_i})$ can be computed as,
\begin{equation}
    (n-1)\int_0^{\infty}{\left[\int_{B(x_i,r)\cap{C^-}}f(y)dy\right]dF_{R_{x_i}^k}(r)}.
\end{equation}
where $r$ is the distance of $x_i$ to its $k_n\rho(x_i)$-th nearest neighbor. The value of $r$ is a random variable and can be characterized by the CDF $F_{R_{x_i}^k}(r)$.
Combining equation \ref{eq:expect} we can write down the whole expected cut value
\begin{eqnarray}
  \mathbb{E}(cut_n(S)) =n\mathbb{E}_x(\mathbb{E}(cut_{x}))= n\int_{\mathbb{R}^d}f(x)\mathbb{E}(cut_{x})dx \\
   = n(n-1)\int_{\mathbb{R}^d}f(x)\left[\int_0^{\infty}{g(x,r)dF_{R_x^k}(r)}\right]dx.
\end{eqnarray}

To simplify the expression, we use $g(x,r)$ to denote
\begin{equation}
    g(x,r)=\begin{cases}
               \int_{B(x,r)\cap{C^-}}f(y)dy,  x\in{C^+} \\
               \int_{B(x,r)\cap{C^+}}f(y)dy,  x\in{C^-}.
             \end{cases}
\end{equation}

Under general assumptions, when $n$ tends to infinity, the random variable $r$ will highly concentrate around its mean $\mathbb{E}(r_x^k)$.
Furthermore, as $k_n/n\rightarrow{0}$, $\mathbb{E}(r_x^k)$ tends to zero and the speed of convergence
\begin{eqnarray}\label{eq:EkNN}
\mathbb{E}(r_x^k)\approx(k\rho(x)/((n-1)f(x)\eta_d))^{1/d}
\end{eqnarray}
So the inner integral in the cut value can be approximated by $g(x,\mathbb{E}(r_x^k))$, which implies,
\begin{equation}
    \mathbb{E}(cut_n(S))\approx{n}(n-1)\int_{\mathbb{R}^d}f(x)g(x,\mathbb{E}(r_x^k))dx.
\end{equation}

The next trick is to decompose the integral over $\mathbb{R}^d$ into two orthogonal directions, i.e., the direction along the hyperplane $S$ and its normal direction (We use $\overrightarrow{n}$ to denote the unit normal vector):
\begin{equation}
    \int_{\mathbb{R}^d}f(x)g(x,\mathbb{E}(r_x^k))dx= \\
    \int_{S}\int_{-\infty}^{+\infty}f(s+t\overrightarrow{n})g(s+t\overrightarrow{n},\mathbb{E}(r_{s+t\overrightarrow{n}}^k))dtds.
\end{equation}
When $t>\mathbb{E}(r_{s+t\overrightarrow{n}}^k)$, the integral region of $g$ will be empty: $B(x,\mathbb{E}(r_x^k))\cap{C^-}=\emptyset$. On the other hand, when $x=s+t\overrightarrow{n}$ is close to $s\in{S}$, we have the approximation $f(x)\approx{f(s)}$:
\begin{eqnarray}
  &\int_{-\infty}^{+\infty}f(s+t\overrightarrow{n})g(s+t\overrightarrow{n},\mathbb{E}(r_{s+t\overrightarrow{n}}^k))dt \\
  &\approx 2\int_{0}^{\mathbb{E}(r_{s}^k)}f(s)\left[f(s)vol\left(B(s+t\overrightarrow{n},\mathbb{E}{r_s^k})\cap{C^-}\right)\right]dt  \\
  &= 2f^2(s)\int_{0}^{\mathbb{E}(r_{s}^k)}vol\left(B(s+t\overrightarrow{n},\mathbb{E}(r_s^k))\cap{C^-}\right)dt.
\end{eqnarray}

The term $vol\left(B(s+t\overrightarrow{n},\mathbb{E}(r_s^k))\cap{C^-}\right)$ is the volume of $d$-dim spherical cap of radius $\mathbb{E}(r_s^k))$, which is at distance $t$ to the center. Through direct computation we obtain:
\begin{equation}
    \int_{0}^{\mathbb{E}(r_{s}^k)}vol\left(B(s+t\overrightarrow{n},\mathbb{E}(r_s^k))\cap{C^-}\right)dt=\mathbb{E}(r_s^k)^{d+1}\frac{\eta_{d-1}}{d+1}.
\end{equation}
Combining the above step and plugging in the approximation of $\mathbb{E}(r_s^k)$ in Eq.(\ref{eq:EkNN}), we finish the proof.
\end{proof}

\begin{lem}\label{lem:expectation}
By choosing $l$ properly, as $m_2\rightarrow\infty$, it follows that,
$$ | \mathbb{E}\left[R(u)\right] - p(u)| \longrightarrow 0$$
\end{lem}
\begin{proof}
Take expectation with respect to $D$:
\begin{eqnarray}
\mathbb{E}_D\left[R(u)\right]
&=&\mathbb{E}_{D\backslash D_0}\left[\mathbb{E}_{D_0}\left[\frac{1}{m_1}\sum_{j=1}^{m_1}
 \mathbb{I}_{\{G(u;D_j)<G(x_j;D_j)\}}\right]\right]\\
&=&\frac{1}{m_1}\sum_{j=1}^{m_1}\mathbb{E}_{x_j}\left[
\mathbb{E}_{D_j}\left[
\mathbb{I}_{\{G(u;D_j)<G(x_j;D_j)\}}\right]\right]\\
&=&\mathbb{E}_x\left[\mathcal{P}_{D_1}\left(G(u;D_1)<G(x;D_1)\right)\right]
\end{eqnarray}
The last equality holds due to the i.i.d symmetry of $\{x_1,...,x_{m_1}\}$ and $D_1,...,D_{m_1}$. We fix both $u$ and $x$ and temporarily discarding $\mathbb{E}_{D_1}$. Let $F_x(y_1,...,y_{m_2})=G(x)-G(u)$, where $y_1,...,y_{m_2}$ are the $m_2$ points in $D_1$. It follows:
\begin{equation}
    \mathcal{P}_{D_1}\left(G(u)<G(x)\right)
    =\mathcal{P}_{D_1}\left(F_x(y_1,...,y_{m_2})>0\right)
    =\mathcal{P}_{D_1}\left(F_x-\mathbb{E}F_x>-\mathbb{E}F_x\right).
\end{equation}

To check McDiarmid's requirements, we replace $y_j$ with $y_j'$. It is easily verified that $\forall j=1,...,m_2$,
\begin{equation}\label{equ:mcdiarmid_condition}
    |F_x(y_1,...,y_{m_2})-F_x(y_1,...,y_j',...,y_{m_2})| \leq 2^{\frac{1}{d}}\frac{2C}{l} \leq \frac{4C}{l}
\end{equation}
where $C$ is the diameter of support. Notice despite the fact that $y_1,...,y_{m_2}$ are random vectors we can still apply MeDiarmid's inequality, because according to the form of $G$, $F_x(y_1,...,y_{m_2})$ is a function of $m_2$ i.i.d random variables $r_1,...,r_{m_2}$ where $r_i$ is the distance from $x$ to $y_i$.
Therefore if $\mathbb{E}F_x<0$, or $\mathbb{E}G(x)<\mathbb{E}G(u)$, we have by McDiarmid's inequality,
\begin{equation}
    \mathcal{P}_{D_1}\left(G(u)<G(x)\right)
    = \mathcal{P}_{D_1}\left( F_x > 0 \right)
    = \mathcal{P}_{D_1}\left( F_x-\mathbb{E}F_x>-\mathbb{E}F_x \right)
    \leq \exp\left(-\frac{(\mathbb{E}F_x)^2 l^2}{8C^2m_2}\right)
\end{equation}
Rewrite the above inequality as:
\begin{equation}\label{equ:bound_no_expectation}
    \mathbb{I}_{\{\mathbb{E}F_x>0\}}-e^{-\frac{(\mathbb{E}F_x)^2 l^2}{8C^2m_2}}
    \leq \mathcal{P}_{D_1}\left( F_x > 0 \right)
    \leq \mathbb{I}_{\{\mathbb{E}F_x>0\}}+e^{-\frac{(\mathbb{E}F_x)^2 l^2}{8C^2m_2}}
\end{equation}
It can be shown that the same inequality holds for $\mathbb{E}F_x>0$, or $\mathbb{E}G(x)>\mathbb{E}G(u)$. Now we take expectation with respect to $x$:
\begin{equation}\label{equ:bound_with_expectation}
    \mathcal{P}_x\left(\mathbb{E}F_x>0\right)-\mathbb{E}_x\left[e^{-\frac{(\mathbb{E}F_x)^2 l^2}{8C^2m_2}}\right] \leq
    \mathbb{E}\left[\mathcal{P}_{D_1}\left( F_x > 0 \right)\right] \leq \mathcal{P}_x\left(\mathbb{E}F_x>0\right)+\mathbb{E}_x\left[e^{-\frac{(\mathbb{E}F_x)^2 l^2}{8C^2m_2}}\right]
\end{equation}
Divide the support of $x$ into two parts, $\mathbb{X}_1$ and $\mathbb{X}_2$, where $\mathbb{X}_1$ contains those $x$ whose density $f(x)$ is relatively far away from $f(u)$, and $\mathbb{X}_2$ contains those $x$ whose density is close to $f(u)$. We show for $x \in \mathbb{X}_1$, the above exponential term converges to 0 and $\mathcal{P}\left(\mathbb{E}F_x>0\right) = \mathcal{P}_x\left( f(u)>f(x) \right)$, while the rest $x\in\mathbb{X}_2$ has very small measure. Let $A(x)=\left(\frac{k}{f(x) c_d m_2}\right)^{1/d}$. By Lemma \ref{lem:bound_expectation} we have:
\begin{equation}
    | \mathbb{E}G(x) - A(x) | \leq \gamma \left(\frac{l}{m_2}\right)^{\frac{1}{d}} A(x)
    \leq \gamma \left(\frac{l}{m_2}\right)^{\frac{1}{d}} \left(\frac{l}{f_{min}c_d m_2}\right)^{\frac{1}{d}}
    =    \left(\frac{\gamma_1}{c_d^{1/d}}\right) \left(\frac{l}{m_2}\right)^{\frac{2}{d}}
\end{equation}
where $\gamma$ denotes the big $O(\cdot)$, and $\gamma_1 = \gamma \left(\frac{1}{f_{min}}\right)^{1/d}$. Applying uniform bound we have:
\begin{equation}
    A(x)-A(u)- 2\left(\frac{\gamma_1}{c_d^{1/d}}\right) \left(\frac{l}{m_2}\right)^{\frac{2}{d}}
    \leq \mathbb{E}\left[G(x) - G(u)\right]
    \leq A(x)-A(u)+ 2\left(\frac{\gamma_1}{c_d^{1/d}}\right) \left(\frac{l}{m_2}\right)^{\frac{2}{d}}
\end{equation}
Now let $\mathbb{X}_1=\{ x:|f(x)-f(u)|\geq 3\gamma_1 d f_{min}^{\frac{d+1}{d}} \left(\frac{l}{m_2}\right)^{\frac{1}{d}} \}$. For $x\in \mathbb{X}_1$, it can be verified that $|A(x)-A(u)|\geq 3\left(\frac{\gamma_1}{c_d^{1/d}}\right) \left(\frac{l}{m_2}\right)^{\frac{2}{d}}$, or $|\mathbb{E}\left[G(x) - G(u)\right]| > \left(\frac{\gamma_1}{c_d^{1/d}}\right) \left(\frac{l}{m_2}\right)^{\frac{2}{d}}$, and $\mathbb{I}_{\{f(u)>f(x)\}}=\mathbb{I}_{\{\mathbb{E}G(x)>\mathbb{E}G(u)\}}$. For the exponential term in Equ.(\ref{equ:bound_no_expectation}) we have:
\begin{equation}
    \exp\left(-\frac{(\mathbb{E}F_x)^2 l^2}{2C^2m_2}\right)
    \leq \exp\left(-\frac{ \gamma_1^2 l^{2+\frac{4}{d}} }{ 8C^2 c_d^{\frac{2}{d}} m_2^{1+\frac{4}{d}} } \right)
\end{equation}
For $x\in \mathbb{X}_2=\{x:|f(x)-f(u)|< 3\gamma_1 d \left(\frac{l}{m_2}\right)^{\frac{1}{d}}f_{min}^{\frac{d+1}{d}} \}$, by the regularity assumption, we have $\mathcal{P}(\mathbb{X}_2)<3M\gamma_1 d \left(\frac{l}{m_2}\right)^{\frac{1}{d}}f_{min}^{\frac{d+1}{d}}$. Combining the two cases into Equ.(\ref{equ:bound_with_expectation}) we have for upper bound:
\begin{eqnarray}
  \mathbb{E}_D\left[R(u)\right]
  &=& \mathbb{E}_x\left[\mathcal{P}_{D_1}\left(G(u)<G(x)\right)\right] \\
  &=& \int_{\mathbb{X}_1}\mathcal{P}_{D_1}\left(G(u)<G(x)\right)f(x)dx +  \int_{\mathbb{X}_2}\mathcal{P}_{D_1}\left(G(u)<G(x)\right)f(x)dx \\
  &\leq& \left( \mathcal{P}_x\left(f(u)>f(x)\right) + \exp\left(-\frac{ \gamma_1^2 l^{2+\frac{4}{d}} }{ 8C^2 c_d^{\frac{1}{d}} m_2^{1+\frac{4}{d}} } \right) \right)\mathcal{P}(x\in \mathbb{X}_1) + \mathcal{P}(x\in \mathbb{X}_2) \\
  &\leq&  \mathcal{P}_x\left(f(u)>f(x)\right) + \exp\left(-\frac{ \gamma_1^2 l^{2+\frac{4}{d}} }{ 8C^2 c_d^{\frac{1}{d}} m_2^{1+\frac{4}{d}} } \right) + 3M\gamma_1 d f_{min}^{\frac{d+1}{d}} \left(\frac{l}{m_2}\right)^{\frac{1}{d}}
\end{eqnarray}
Let $l=m_2^\alpha$ such that $\frac{d+4}{2d+4}<\alpha<1$, and the latter two terms will converge to 0 as $m_2 \rightarrow \infty$. Similar lines hold for the lower bound. The proof is finished.
\end{proof}

\begin{lem}\label{lem:bound_expectation}
Let $A(x)=\left(\frac{l}{m c_d f(x)}\right)^{1/d}$, $\lambda_1 = \frac{\lambda}{f_{min}}\left(\frac{1.5}{c_d f_{min}}\right)^{1/d}$. By choosing $l$ appropriately, the expectation of $l$-NN distance $\mathbb{E}D_{(l)}(x)$ among $m$ points satisfies:
\begin{equation}
    | \mathbb{E}D_{(l)}(x) - A(x) | = O\left( A(x) \lambda_1 \left(\frac{l}{m}\right)^{1/d} \right)
\end{equation}
\end{lem}

\begin{proof}
Denote $r(x,\alpha)=\min\{r:\mathcal{P}\left(B(x,r)\right)\geq \alpha\}$. Let $\delta_m \rightarrow 0$ as $m \rightarrow \infty$, and $0<\delta_{m}<1/2$.
Let $U\sim Bin(m,(1+\delta_m)\frac{l}{m})$ be a binomial random variable, with $\mathbb{E}U = (1+\delta_{m})l$. We have:
\begin{eqnarray}
  \mathcal{P}\left(D_{(l)}(x)>r(x,(1+\delta_m)\frac{l}{m})\right)
  &=& \mathcal{P}\left(U<l\right) \\
  &=& \mathcal{P}\left(U<\left(1-\frac{\delta_m}{1+\delta_m}\right)(1+\delta_m)l\right) \\
  &\leq& \exp\left(-\frac{\delta_m^2 l}{2(1+\delta_m)}\right)
\end{eqnarray}
The last inequality holds from Chernoff's bound. Abbreviate $r_1 = r(x,(1+\delta_m)\frac{l}{m})$, and $\mathbb{E}D_{(l)}(x)$ can be bounded as:
\begin{eqnarray}
  \mathbb{E}D_{(l)}(x)
  &\leq& r_1\left[1-\mathcal{P}\left(D_{(l)}(x)>r_1\right)\right] + C\mathcal{P}\left(D_{(l)}(x)>r_1\right)  \\
  &\leq& r_1 + C \exp\left(-\frac{\delta_m^2 l}{2(1+\delta_m)}\right)
\end{eqnarray}
where $C$ is the diameter of support. Similarly we can show the lower bound:
\begin{equation}
    \mathbb{E}D_{(l)}(x) \geq r(x,(1-\delta_m)\frac{l}{m}) - C \exp\left(-\frac{\delta_m^2 l}{2(1-\delta_m)}\right)
\end{equation}
Consider the upper bound. We relate $r_1$ with $A(x)$. Notice $\mathcal{P}\left(B(x,r_1)\right)=(1+\delta_m)\frac{l}{m} \geq c_d r_1^d f_{min}$, so a fixed but loose upper bound is $r_1 \leq \left(\frac{(1+\delta_m)l}{c_d f_{min} m}\right)^{1/d} = r_{max}$. Assume $l/m$ is sufficiently small so that $r_1$ is sufficiently small. By the smoothness condition, the density within $B(x,r_1)$ is lower-bounded by $f(x)-\lambda r_1$, so we have:
\begin{eqnarray}
  \mathcal{P}\left(B(x,r_1)\right) &=& (1+\delta_m)\frac{l}{m} \\
  &\geq& c_d r_1^d \left( f(x)-\lambda r_1 \right)\\
  &=& c_d r_1^d f(x)\left( 1-\frac{\lambda}{f(x)}r_1 \right) \\
  &\geq& c_d r_1^d f(x)\left( 1-\frac{\lambda}{f_{min}}r_{max} \right)
\end{eqnarray}
That is:
\begin{equation}
    r_1 \leq A(x)\left( \frac{1+\delta_m}{1-\frac{\lambda}{f_{min}}r_{max}} \right)^{1/d}
\end{equation}
Insert the expression of $r_{max}$ and set $\lambda_1 = \frac{\lambda}{f_{min}}\left(\frac{1.5}{c_d f_{min}}\right)^{1/d}$, we have:
\begin{eqnarray}
  \mathbb{E}D_{(l)}(x)-A(x) &\leq& A(x)\left( \left(\frac{1+\delta_m}{1-\lambda_1 \left(\frac{l}{m}\right)^{1/d}}\right)^{1/d} -1 \right) + C \exp\left(-\frac{\delta_m^2 l}{2(1+\delta_m)}\right) \\
  &\leq& A(x)\left( \frac{1+\delta_m}{1-\lambda_1 \left(\frac{l}{m}\right)^{1/d}}-1 \right) + C \exp\left(-\frac{\delta_m^2 l}{2(1+\delta_m)}\right) \\
  &=& A(x)\frac{\delta_m + \lambda_1 \left(\frac{l}{m}\right)^{1/d}}{1-\lambda_1\left(\frac{l}{m}\right)^{1/d}} + C \exp\left(-\frac{\delta_m^2 l}{2(1+\delta_m)}\right) \\
  &=& O\left( A(x) \lambda_1 \left(\frac{l}{m}\right)^{1/d} \right)
\end{eqnarray}
The last equality holds if we choose $l=m^{\frac{3d+8}{4d+8}}$ and $\delta_m=m^{-\frac{1}{4}}$. Similar lines follow for the lower bound. Combine these two parts and the proof is finished.

\end{proof}

%
%
%
%

\end{document}